\documentclass{article}

\makeatletter
\def\input@path{{styles/}}
\makeatother
\usepackage[preprint]{colm2026_conference}
\usepackage{fontspec}

\normalfont
\usepackage{microtype}
\usepackage{graphicx}
\usepackage{trimclip}
\usepackage{xcolor}
\usepackage{booktabs}
\usepackage{array}
\usepackage{fontawesome5}
\usepackage{hyperref}
\usepackage{colortbl}
\usepackage{float}
\usepackage{tikz}
\usepackage{tcolorbox}
\usepackage{pgfplots}
\pgfplotsset{compat=1.18}
\usepgfplotslibrary{groupplots}
\usepackage{hyperref}
\usepackage{url}
\usepackage{wrapfig}
\usepackage{makecell}
\usepackage{enumitem}
\usepackage{algorithm}
\usepackage{amsmath}
\usepackage{algpseudocode}
\usepackage{amssymb}
\DeclareMathOperator*{\argmax}{arg\,max}
\newtheorem{proposition}{Proposition}
\newenvironment{proof}{\noindent\textit{Proof.}\enspace}{\hfill$\square$\medskip}
\definecolor{abyss}{HTML}{121D36}
\definecolor{polarnight}{HTML}{1A2947}
\definecolor{nebula}{HTML}{2B3F66}
\definecolor{steeltrail}{HTML}{6D87BD}
\definecolor{skytrail}{HTML}{8FA8D8}
\definecolor{starlight}{HTML}{DFE7F5}
\definecolor{warmstar}{HTML}{E8D9C4}
\definecolor{allsparkwordmark}{HTML}{16233F}
\definecolor{allsparkspark}{HTML}{4A659C}
\definecolor{electricblue}{HTML}{3866FF}
\definecolor{covercream}{HTML}{EEF3FA}
\definecolor{coveraccent}{HTML}{3866FF}
\colorlet{pevekpurple}{skytrail}
\colorlet{bargray}{steeltrail}
\colorlet{barlgray}{starlight}

\newfontfamily\outfit[
  Path=assets/fonts/,
  UprightFont=Outfit-Regular.ttf,
  BoldFont=Outfit-SemiBold.ttf
]{Outfit}

\hypersetup{
  colorlinks=true,
  linkcolor=electricblue,
  citecolor=electricblue,
  urlcolor=coveraccent,
  filecolor=electricblue
}
\setcitestyle{numbers,square,comma,sort&compress}

\newcommand{\reporttitle}{%
  \textbf{D}\textsuperscript{3}\text{-MOPD:}\ \textbf{D}ynamic \textbf{D}omain Sche\textbf{D}uling for Efficient Multi-Teacher Distillation}
\title{\reporttitle}
\author{Anonymous Authors}

\begin{document}

\fancyhead{}
\renewcommand{\headrulewidth}{0pt}
\color{abyss}
\thispagestyle{empty}

\vspace*{-0.44in}
\begin{tcolorbox}[
  width=\linewidth,
  colback=covercream,
  colframe=covercream,
  boxrule=0pt,
  arc=14pt,
  outer arc=14pt,
  boxsep=0pt,
  left=20pt,
  right=20pt,
  top=13pt,
  bottom=11pt
]
  {\outfit\fontsize{20}{24}\selectfont\bfseries\centering
    \reporttitle\par}

  \vspace{1.45em}
  {\bfseries\centering AllSpark Team\par}

  \vspace{0.75em}
  \begingroup
  \normalfont
  \setlength{\parindent}{0pt}
  \setlength{\parskip}{0pt}
  Multi-teacher on-policy distillation (MOPD) distills several domain-expert teachers into a single student by minimizing per-domain reverse-KL divergence on the student's own rollouts. Existing approaches typically fix the per-domain data mixture before training, overlooking the fact that different domains converge at substantially different rates: some plateau early while others continue to improve throughout the training budget. A fixed mixture therefore wastes compute on fast-converging domains and undertrains slower-converging ones. To address this, we propose D$^3$-MOPD (\textbf{D}ynamic \textbf{D}omain Sche\textbf{D}uling for MOPD), a zero-overhead scheduler that repurposes the per-domain reverse-KL signal already produced during training to adapt the domain mixture online. Running asynchronously outside the training process, an off-process watcher periodically tracks each domain's KL trajectory, estimates remaining headroom and current improvement rate, and accordingly adjusts the domain sampling ratios without altering the core training loop. Our D$^3$-MOPD scales naturally to arbitrary numbers of domains, and the expected benefit grows as more domains introduce more diverse convergence patterns for the scheduler to exploit. On a Qwen3.6-35B-A3B student distilled from four domain-expert teachers, D$^3$-MOPD closes $97\%$ of the average student-to-teacher performance gap, compared with $63\%$ for vanilla MOPD, reaches the same peak performance with an approximately $3\times$ reduction in rollout steps, and surpasses the specialist teachers on three of seven benchmarks.

  \par
  \endgroup

  \vspace{0.65em}
  \noindent
  \begin{minipage}[b]{0.63\linewidth}
    \outfit\fontsize{8.4}{10.2}\selectfont
    \textbf{Date:} September 15, 2026\\[-0.1em]
    \textbf{Resources:}
    \href{https://github.com/AllSpark-Research/D-3-MOPD}{GitHub}\
  \end{minipage}%
  \hfill
  \begin{minipage}[b]{0.33\linewidth}
    \raggedleft
    \raisebox{-0.30em}{\includegraphics[height=16pt]{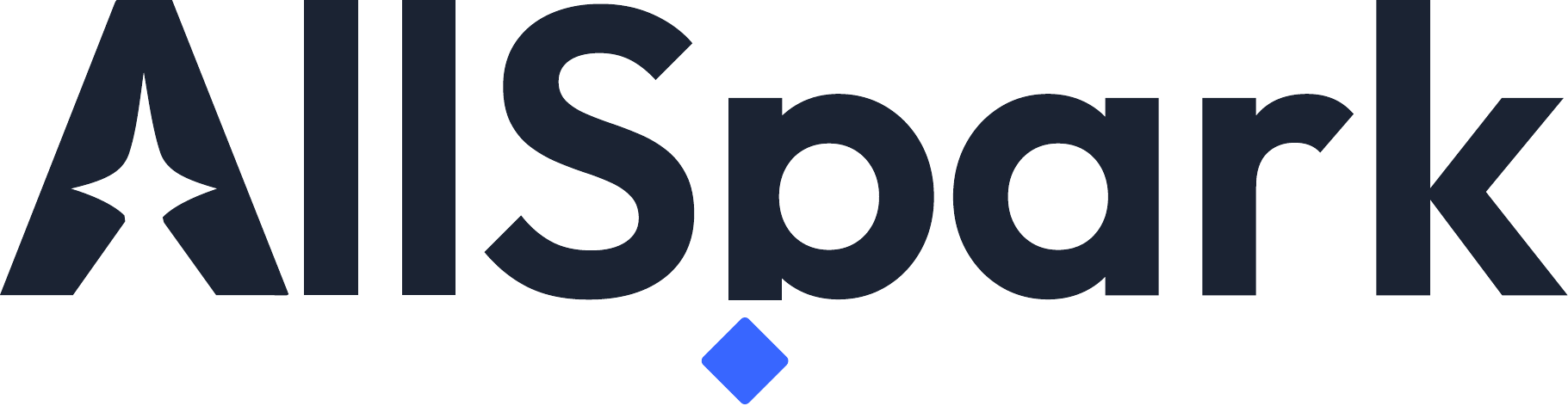}}%
  \end{minipage}
\end{tcolorbox}

\vspace{0.6em}

\vspace{-0.6em}
\section{Introduction}
\label{sec:intro}

\vspace{-0.65em}

\begin{wrapfigure}{r}{0.4\linewidth}
  \centering
  \vspace{-12pt}
  \includegraphics[width=\linewidth]{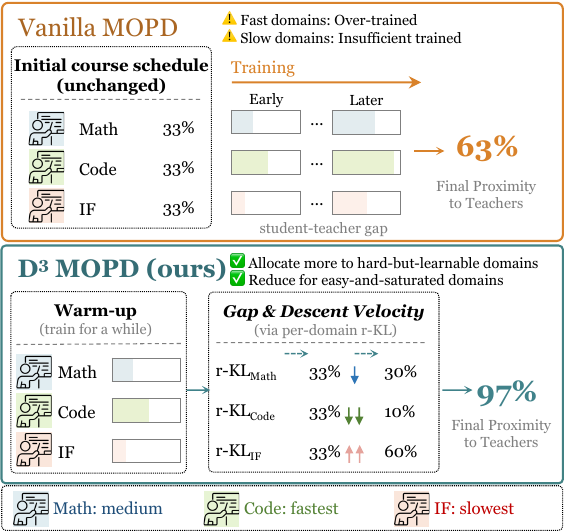}
  \vspace{-15pt}
  \caption{Comparison of vanilla MOPD and D$^3$-MOPD with $K=3$ domains.}
  \label{fig:intro}
  \vspace{-10pt}
\end{wrapfigure}

Multi-teacher on-policy distillation (MOPD) has recently emerged as an effective framework for combining the capabilities of several domain-expert teacher models into a single student~\citep{coreteam2026mimov2flashtechnicalreport,deepseekai2026deepseekv4highlyefficientmilliontoken}.
At each rollout step, the student generates on-policy trajectories from a $K$-domain prompt set~\citep{yang2026nemotron,zeng2026glm,m3team2026baichuanm3modelingclinicalinquiry}, and each domain teacher prefills on the corresponding trajectories to produce its per-token distribution. The student then minimizes the per-token reverse-KL divergence~\citep{agarwal2024policy,gu2024minillm} between its own distribution and the teacher's along the trajectory. The result is a unified model that inherits the strengths of all teachers, without requiring all of them to be served at inference.

Current MOPD implementations typically rely on a static domain mixture $\{p_k\}$, fixed before training to match the prompt pool sizes.
However, the student rarely learns each domain at the same rate~\citep{chen2026counteraction,ma2026mopd}: some domains plateau early as the student quickly closes the performance gap to the teacher, whereas others remain far from convergence throughout the entire training budget.
Consequently, a fixed mixture wastes significant compute on converged domains while limiting the training budget for actively learning domains (Figure~\ref{fig:intro}).
Since the per-domain reverse-KL is already computed by the loss objective at every rollout step, a natural and promising solution is to repurpose it as a readily available signal for identifying and correcting this training imbalance.

In this paper, we propose \emph{\textbf{D}ynamic \textbf{D}omain Sche\textbf{D}uling for MOPD} (D$^3$-MOPD), a lightweight scheduler that repurposes the per-domain reverse-KL signal already produced by the training loop to update the mixture $\{p_k\}$ online.
An asynchronous watcher periodically converts each domain's KL history into a \emph{composite signal} that combines its remaining gap to the teacher with its recent descent velocity, and maps this signal to a valid mixture via a temperature-controlled softmax with a per-domain floor. It then passes the updated mixture to the stratified data source, ensuring the core training loop remains strictly unmodified.
By confining modifications to the data loader and running the watcher detached, D$^3$-MOPD remains compatible with advances in loss functions~\citep{yang2026learningteachergeneralizedonpolicy,zhang-etal-2026-fast,jang-etal-2026-stable} and teacher training~\citep{song2026surveyonpolicydistillationlarge,fang2026flowopdonpolicydistillationflow}, without additional training overhead.

We validate D$^3$-MOPD on a Qwen3.6-35B-A3B~\citep{qwen36_35b_a3b} student distilled from four domain-expert teachers spanning Math, Code, Instruction Following, and Tool-use. Experimental results highlight that D$^3$-MOPD outperforms vanilla MOPD, closing $97\%$ of the average student-to-teacher performance gap (compared to $63\%$ for the baseline) while achieving a higher peak accuracy across all tasks, surpassing the specialist teachers on three of seven benchmarks. It also improves training efficiency by reaching the baseline's optimal performance in just $47$ rollout steps, approximately $3\times$ faster than vanilla MOPD ($143$ steps). Furthermore, we argue that the expected benefits of D$^3$-MOPD will naturally scale with $K$, as more domains introduce more diverse convergence patterns for the scheduler to exploit. Overall, this paper first identifies a systematic mismatch between MOPD's static mixture and the student's asynchronous per-domain convergence, and resolves it with D$^3$-MOPD, a zero-overhead scheduler that turns the reverse-KL signal already computed by the training loop into a dynamic data mixture to improve both peak quality and training efficiency.

\vspace{-0.5em}
\section{Motivation: Limitations of Fixed Domain Mixtures}
\label{sec:motivation}

In this section, we begin by formalizing vanilla MOPD and its fixed mixture (\S\ref{ssec:formulation}), then analyze per-domain learning dynamics in isolation (\S\ref{ssec:dynamics}), and finally demonstrate how their disparate convergence rates waste compute under a static mixture (\S\ref{ssec:inefficiency}).

\subsection{Vanilla MOPD and Fixed Mixtures}
\label{ssec:formulation}
On-policy distillation (OPD) trains a student policy $\pi_\theta$ on its own generations, supervising them with a teacher $\pi_T$ under a reverse Kullback-Leibler (KL) objective. Multi-teacher OPD (MOPD) extends this paradigm to $K$ domains, where each domain is associated with a prompt distribution $\mathcal{D}_k$ and an expert teacher $\pi_{T_k}$. MOPD minimizes a weighted sum of per-domain losses:
\begin{equation}
    \mathcal{J}_{\text{MOPD}}(\theta) \;=\; \sum_{k=1}^{K} p_k \, \mathbb{E}_{x \sim \mathcal{D}_k,\; y \sim \pi_\theta(\cdot \mid x)} \Big[ D_{\mathrm{KL}}\big(\pi_\theta(\cdot \mid x, y) \,\|\, \pi_{T_k}(\cdot \mid x, y)\big) \Big],
\end{equation}
where the KL divergence is summed over the generated tokens of $y$, and the mixture weight $p_k$ satisfies $p_k \ge 0$ with $\sum_k p_k = 1$. In vanilla MOPD, the per-domain datasets $\mathcal{D}_k$ are pooled and shuffled, resulting in an expected per-batch share for domain $k$ of $p_k = N_k / \sum_i N_i$, where $N_k = |\mathcal{D}_k|$. Crucially, this expected share remains strictly defined prior to training and is held constant throughout the run. Pool-and-shuffle sampling fixes only this expectation, and each batch's realized share fluctuates around $p_k$.

\subsection{Domain-Specific Learning Dynamics}
\label{ssec:dynamics}
To understand the inefficiency of a static $p_k$, we analyze each domain. Specifically, we train the student separately on Math, Code, and Instruction Following (IF) using single-domain OPD, with identical initializations and the corresponding expert teacher for each domain. Figure~\ref{fig:sopd_curves} reports the per-step reverse-KL on a held-out evaluation set. Math drops sharply during the first quarter of the training budget before plateauing near a floor. Code declines steadily at a slower rate, showing no signs of convergence within the budget. Meanwhile, IF continues to reduce its loss throughout, yet its absolute reverse-KL remains one to two orders of magnitude above the other domains at every step. In conclusion, these domains reach their plateau phases at different stages, with their absolute KL signals varying by orders of magnitude.

\begin{figure}[t]
  \centering
  \includegraphics[width=\linewidth]{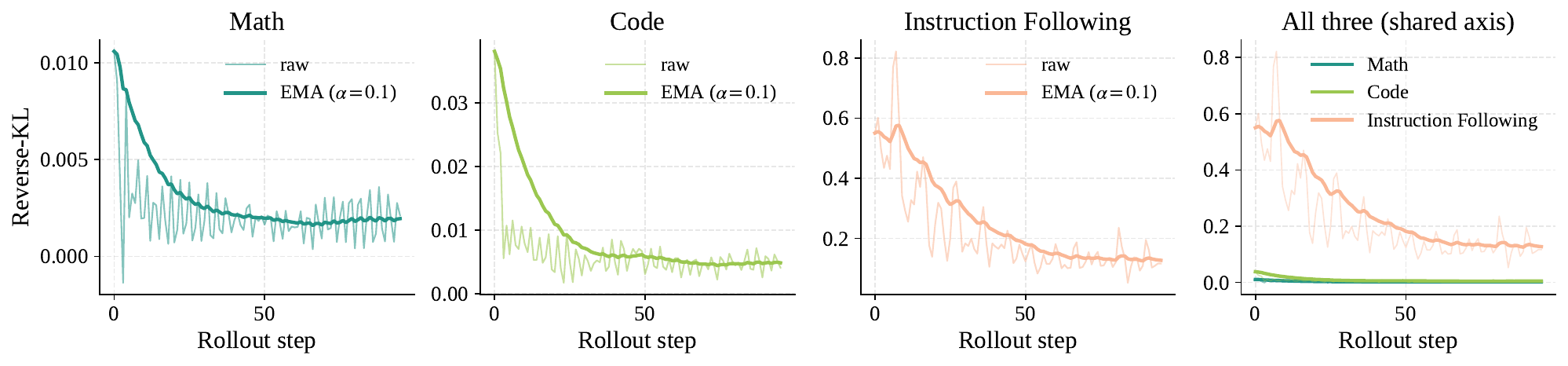}
  \vspace{-20pt}
  \caption{Per-domain OPD reverse-KL on Math, Code, and IF. The three domains converge at widely different rates, with the absolute KL of IF remaining one to two orders of magnitude above the others throughout the training process.}
  \vspace{-10pt}
  \label{fig:sopd_curves}
\end{figure}

\subsection{Computational Inefficiency of Fixed Mixtures}
\label{ssec:inefficiency}
When these same three domains are trained jointly under vanilla MOPD with a uniform mixture ($p_{\text{math}} = p_{\text{code}} = p_{\text{if}} = 1/3$), the varying convergence rates documented in \S\ref{ssec:dynamics} translate into a substantial waste of training compute. As illustrated in Figure~\ref{fig:mopd_fixed}, the three domains enter their low-KL phases at markedly different points: Code by roughly step 48, Math by step 96, and IF only around step 144. Yet the fixed mixture allocates one-third of every batch to each domain for the remainder of training. Consequently, valuable training compute is wasted on domains long past the point where additional samples yield further KL reductions. This waste stems not from the data or the teachers themselves, but from holding $p_k$ fixed.

\begin{figure}[t]
  \centering
  \includegraphics[width=\linewidth]{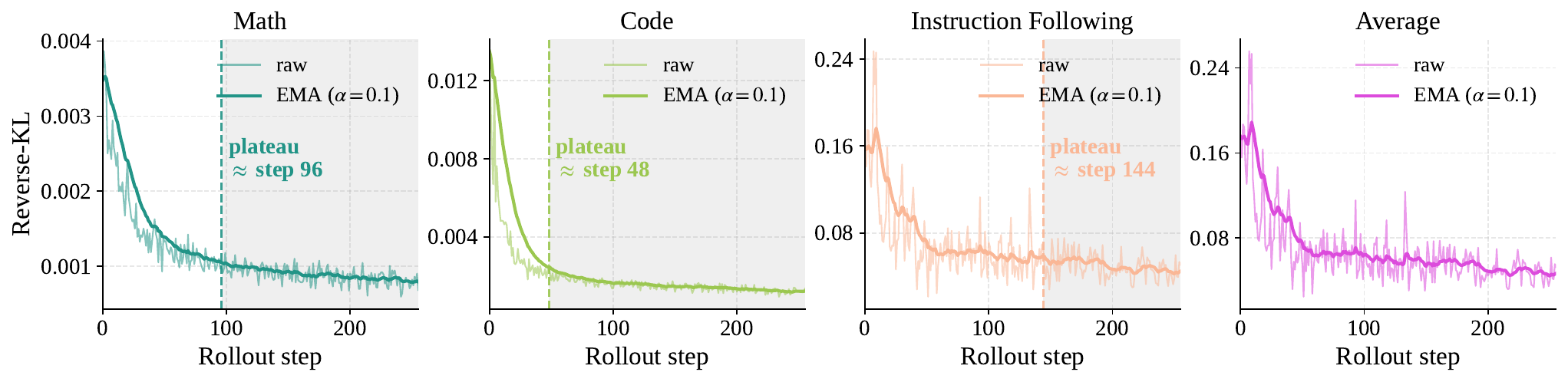}
  \vspace{-20pt}
  \caption{Vanilla MOPD with a fixed uniform mixture ($p_k = 1/3$). The three domains reach their low-KL phases at markedly different points, yet the static mixture allocates one-third of every batch to each throughout the remaining training (shaded regions).}
  \label{fig:mopd_fixed}
  \vspace{-10pt}
\end{figure}

\paragraph{Takeaway.}
Under vanilla MOPD, distinct domains enter low-KL plateau phases at entirely different training stages (\S\ref{ssec:dynamics}). A fixed mixture inevitably wastes compute on domains that have ceased to improve (\S\ref{ssec:inefficiency}). This highlights the necessity for a dynamic sampling ratio $p_k$ that smoothly adapts to native training signals, motivating the design of our method in \S\ref{sec:method}.

\vspace{-0.5em}
\section{Method}
\label{sec:method}

Based on these observations, we propose \textbf{D}ynamic \textbf{D}omain Sche\textbf{D}uling for MOPD (D$^3$-MOPD), a lightweight scheduler that adapts the per-domain sampling ratio $\{p_k\}$ using the reverse-KL signal. \S\ref{ssec:architecture} describes the framework, which consists of an off-process watcher and a stratified data source. \S\ref{ssec:algorithm} then details how the per-domain KL history is converted into a dynamic mixture.

\vspace{-1em}
\subsection{Framework}
\label{ssec:architecture}

D$^3$-MOPD introduces two components to a vanilla MOPD training loop, as illustrated in Figure~\ref{fig:d3_arch}: an \emph{off-process watcher} that computes the per-domain mixture $p_k$ from the reverse-KL signals already produced by the trainer, and a \emph{stratified data source} that reads $p_k$ to assemble each batch with the target mixture. This pipeline generalizes to any number of domains $K$.

\begin{figure}[t]
  \centering
  \includegraphics[width=0.97\linewidth]{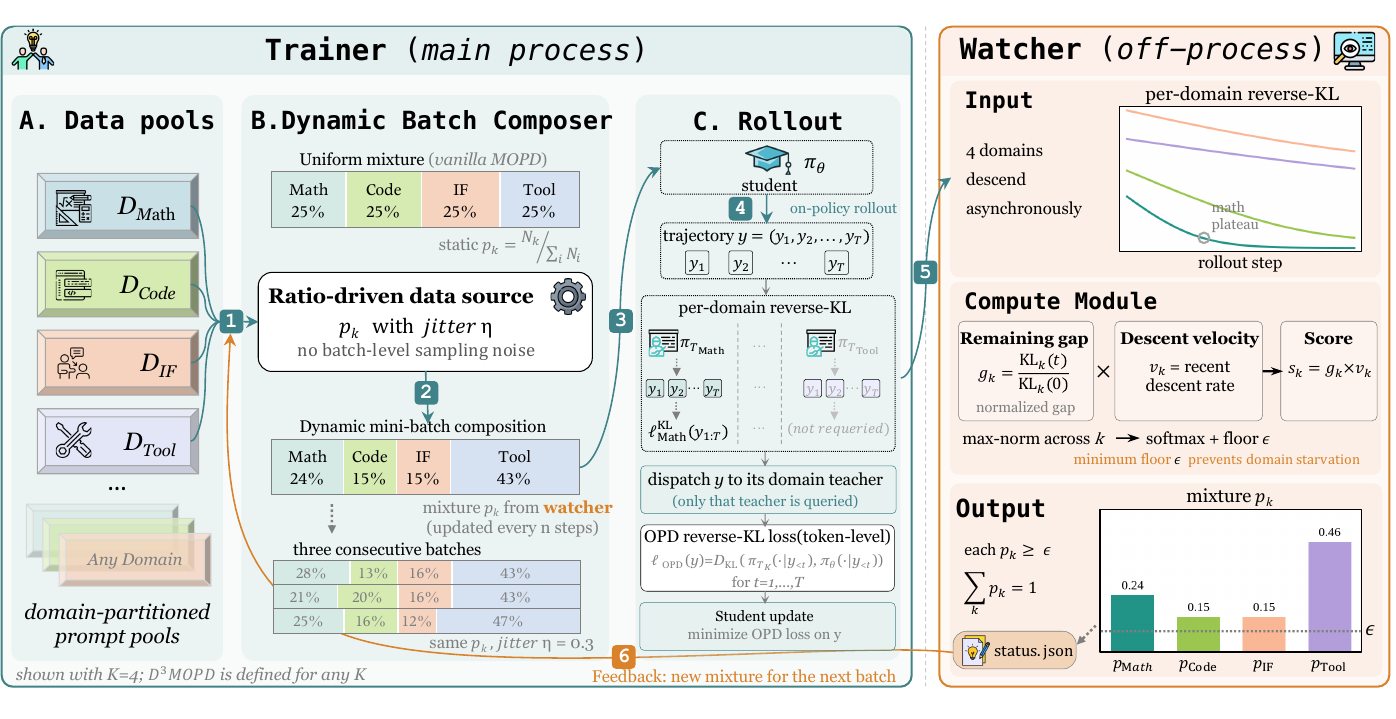}
  \vspace{-10pt}
  \caption{The framework of D$^3$-MOPD. At each rollout step, the student generates responses that the teachers prefill on to produce per-domain reverse-KL signals. A watcher periodically maps these signals into a continuous per-domain mixture $p_k$ that shapes subsequent batches.}
  \label{fig:d3_arch}
\end{figure}

The trainer follows vanilla MOPD: at each rollout step, the student generates responses from prompts drawn by the data source. Each response is dispatched to its corresponding domain teacher, which prefills on the trajectory to produce its per-token distribution. The resulting per-domain reverse-KL is appended to a shared log. Since vanilla MOPD typically logs only the domain-averaged reverse-KL, the watcher groups per-sample KL values by domain to recover per-domain signals. It runs as a separate process that periodically computes a new mixture $p_k$ following \S\ref{ssec:algorithm} and writes it to a status file that the data source reads between updates so that the trainer never blocks.

We replace the conventional pool-and-shuffle loader with the \emph{stratified data source}. Given a batch size $B$ and per-batch mixture $\{\tilde{p}_k\}$ (obtained by applying the batch-level jitter below to the watcher-supplied $\{p_k\}$), it first assigns $\lfloor B \cdot \tilde{p}_k \rfloor$ samples to each domain, then distributes the remaining $B - \sum_k \lfloor B \cdot \tilde{p}_k \rfloor$ samples to the domains with the largest fractional parts. This strictly aligns each domain's per-batch share with the target mixture, replacing the uncontrolled variance of pooled shuffling with the purposeful batch-level jitter described below.

\vspace{-1em}
\paragraph{Batch-level jitter.}
Because the watcher updates the mixture only every $n$ rollout steps, consecutive batches would otherwise share identical domain proportions. The data source therefore perturbs each batch's proportions via $\tilde{p}_k = p_k(1 + u_k) / \sum_j p_j(1 + u_j)$ with $u_k \sim \mathrm{Uniform}(-\eta, \eta)$, restoring batch-to-batch variation while keeping the long-run mean close to $\{p_k\}$ ($\mathbb{E}[u_k] = 0$).

\vspace{-1em}
\paragraph{Integration with existing MOPD.}
All modifications are strictly confined to the data path, leaving the rollout, teacher prefill, and student update operations untouched. An existing MOPD implementation can adopt D$^3$-MOPD by replacing the standard loader with the stratified data source and launching a watcher process alongside the trainer.

\subsection{Composite Ratio from Remaining Gap and Descent Velocity}
\label{ssec:algorithm}

The watcher translates the per-domain reverse-KL history into a mixture $\{p_k\}$ at each update. Our goal is to allocate more samples to domains exhibiting ongoing improvement and fewer to those that have converged. We quantify this improvement along two axes: the \emph{remaining gap} $\widetilde{\mathrm{KL}}_k$ (the proportion of the initial KL left to close) and the \emph{descent velocity} $v_k$ (the recent rate of KL reduction). Relying on either axis alone is insufficient. A domain might exhibit a large remaining gap but plateau, rendering further samples ineffective. Conversely, a nearly converged domain might descend rapidly within the current window, yet its minimal room for improvement makes additional samples redundant. Their product ensures that a domain receives a large share only when it has room to improve and is actively doing so.

\paragraph{Remaining gap.}
We define the remaining gap by normalizing the current smoothed KL by its initial value,
\begin{small}
\begin{equation}
    \widetilde{\mathrm{KL}}_k(t) \;=\; \frac{\overline{\mathrm{KL}}_k(t)}{\mathrm{KL}_k^{(0)}},
    \qquad
    \mathrm{KL}_k^{(0)} \;=\; \frac{1}{S_0}\sum_{t=1}^{S_0} \mathrm{KL}_k(t),
    \label{eq:normalized_kl}
\end{equation}
\end{small}
where $\overline{\mathrm{KL}}_k(t)$ is the exponential moving average (EMA) of domain-$k$ reverse-KL up to rollout step $t$, and $\mathrm{KL}_k^{(0)}$ is the mean of the first $S_0$ observations seeded once at the start of training. Under normal training, $\widetilde{\mathrm{KL}}_k(t) \in (0, 1]$ and approaches zero as the domain converges. As documented in \S\ref{ssec:dynamics}, this normalization enables fair comparison across domains whose absolute KL values may differ by orders of magnitude.

\paragraph{Descent velocity.}
To determine whether a domain's KL continues to decrease, we measure its recent rate of change. Let $W$ be a fixed window length in rollout steps and $R$ the number of non-overlapping windows used for averaging. We define the single-window relative change as
\begin{small}
\begin{equation}
    \delta_k^{(i)}(t) \;=\; \frac{\overline{\mathrm{KL}}_k(t - iW) - \overline{\mathrm{KL}}_k(t - (i+1)W)}{\overline{\mathrm{KL}}_k(t - (i+1)W)}, \qquad i = 0, \dots, R-1,
    \label{eq:delta_window}
\end{equation}
\end{small}
where $\delta_k^{(0)}$ represents the most recent window. The descent velocity takes the negative average of these $R$ changes and clips the result at zero,
\begin{small}
\begin{equation}
    v_k(t) \;=\; \max\!\Big(0,\; -\frac{1}{R}\sum_{i=0}^{R-1} \delta_k^{(i)}(t)\Big),
    \label{eq:descent_velocity}
\end{equation}
\end{small}
meaning $v_k(t) > 0$ strictly requires the domain's KL to decrease on average over the last $R$ windows.

\paragraph{Composite signal.}
We compute the composite signal via the product of these two metrics,
\begin{small}
\begin{equation}
    s_k(t) \;=\; \widetilde{\mathrm{KL}}_k(t)\,\cdot\,v_k(t),
    \qquad
    \tilde{s}_k(t) \;=\; s_k(t) \;\big/\; \max_j s_j(t),
    \label{eq:composite}
\end{equation}
\end{small}
ensuring $s_k$ remains large only when a domain is both far from convergence and steadily descending. The max-normalization brings $\tilde{s}_k$ into $[0, 1]$, making the softmax temperature $T$ independent of the raw signal magnitude. In the degenerate case where all domains plateau and $\max_j s_j(t) = 0$, we set $\tilde{s}_k(t) = 0$ for all $k$, and the mapping defaults to a uniform mixture. Appendix~\ref{app:theory} provides a theoretical analysis showing that under an exponential decay model, $s_k$ approximates the normalized marginal KL reduction, justifying the gap-velocity product as a greedy allocation rule.

\paragraph{Mapping to a valid mixture.}
We map $\tilde{s}_k$ to a probability mixture through a softmax function parameterized by temperature $T$ and a per-domain floor $\epsilon$,
\begin{small}
\begin{equation}
    p_k(t) \;=\; \epsilon \;+\; \bigl(1 - K\epsilon\bigr) \cdot \frac{\exp\bigl(\tilde{s}_k(t) / T\bigr)}{\sum_{j=1}^{K}\exp\bigl(\tilde{s}_j(t) / T\bigr)},
    \label{eq:ratio}
\end{equation}
\end{small}
guaranteeing that $p_k(t) \ge \epsilon$ for every domain and $\sum_k p_k(t) = 1$. The watcher recomputes this mixture at every update and delivers it to the data source via the status file described in \S\ref{ssec:architecture}.

\paragraph{Why average over $R$ windows.}
With $R = 1$, the velocity estimate is highly sensitive to random noise in the KL signal. Because Eq.~\ref{eq:descent_velocity} clips the sum at zero, a short-term KL spike in an actively learning domain would wrongly push $v_k$ to zero, cutting off its sample allocation. Averaging the changes over $R$ non-overlapping windows prevents these brief spikes from accidentally triggering the cutoff. As a result, $v_k$ becomes zero only when the domain remains flat or rises over a longer period, making it a reliable sign of actual convergence.

\vspace{-0.5em}
\paragraph{Warmup.}
The composite signal becomes available once $2W$ observations of $\overline{\mathrm{KL}}_k$ are recorded. At this point, $\widetilde{\mathrm{KL}}_k$ relies on the seeded $\mathrm{KL}_k^{(0)}$, and computing a single $\delta_k^{(0)}$ requires the two window endpoints at $t{-}W$ and $t$. Eq.~\ref{eq:descent_velocity} represents the standard formulation with a full history of $R$ windows. During the warmup phase $2W \le t < (R{+}1)W$, the algorithm replaces $R$ with the available window count $R'(t) = \lfloor t/W \rfloor - 1$. From $t = (R{+}1)W$ onward, $R'(t)$ is capped at $R$, and the standard form of Eq.~\ref{eq:descent_velocity} takes over. Prior to $t = 2W$, the watcher provides no mixture updates, causing the data source to default to a uniform mixture $p_k = 1/K$. The batch-level jitter (\S\ref{ssec:architecture}) remains active throughout this initial period.

\vspace{-0.5em}
\paragraph{Design rationale.}
The temperature $T$ modulates how sharply the sampling share is redirected: as $T \to \infty$, the mixture becomes uniform, whereas as $T \to 0$, the mixture concentrates on the argmax of $\tilde{s}$. Finally, the floor $\epsilon$ prevents any domain from being completely discarded, ensuring already-converged domains retain a minimal presence to mitigate catastrophic forgetting. Algorithm~\ref{alg:d3mopd} summarizes the scheduling process.

\begin{algorithm}[t]
\caption{\textbf{D}ynamic \textbf{D}omain Sche\textbf{D}uling for MOPD (D$^3$-MOPD).}
\label{alg:d3mopd}
\begin{algorithmic}[1]
\small
\Require Domain pools $\{\mathcal{D}_k\}_{k=1}^K$; hyperparameters $T, \epsilon, S_0, R, W, n, \eta, B$; shared KL log and status file.
\State \textbf{Initialize} at step $t = S_0$: $\mathrm{KL}_k^{(0)} \gets \tfrac{1}{S_0}\sum_{\tau=1}^{S_0} \mathrm{KL}_k(\tau)$ for all $k$ \Comment{Eq.~\ref{eq:normalized_kl}; seeded once}
\Function{WatcherTick}{$t$} \Comment{invoked at rollout step $t$, every $n$ steps}
    \If{$t < 2W$} \Return \Comment{initial warmup; data source uses uniform fallback} \EndIf
    \State $R' \gets \min\bigl(R,\, \lfloor t/W \rfloor - 1\bigr)$ \Comment{$R' = R$ once $t \ge (R{+}1)W$}
    \For{$k = 1, \dots, K$}
        \State $\widetilde{\mathrm{KL}}_k \gets \overline{\mathrm{KL}}_k(t) / \mathrm{KL}_k^{(0)}$ \Comment{Eq.~\ref{eq:normalized_kl}}
        \State $\delta_k^{(i)} \gets \bigl[\overline{\mathrm{KL}}_k(t{-}iW) - \overline{\mathrm{KL}}_k(t{-}(i{+}1)W)\bigr] / \overline{\mathrm{KL}}_k(t{-}(i{+}1)W)$ for $i = 0, \dots, R'{-}1$
        \State $v_k \gets \max\!\bigl(0,\, -\tfrac{1}{R'}\sum_i \delta_k^{(i)}\bigr)$ \Comment{Eq.~\ref{eq:descent_velocity}}
        \State $s_k \gets \widetilde{\mathrm{KL}}_k \cdot v_k$
    \EndFor
    \If{$\max_j s_j = 0$} \Comment{degenerate case: all domains plateaued}
        \State $\tilde{s}_k \gets 0$ for all $k$
    \Else
        \State $\tilde{s}_k \gets s_k / \max_j s_j$ for all $k$ \Comment{Eq.~\ref{eq:composite}}
    \EndIf
    \State $p_k \gets \epsilon + (1{-}K\epsilon)\, \dfrac{\exp(\tilde{s}_k/T)}{\sum_j \exp(\tilde{s}_j/T)}$ for all $k$ \Comment{Eq.~\ref{eq:ratio}}
    \State Atomically write $\{p_k\}$ to the status file
\EndFunction
\Function{GetBatch}{$B$} \Comment{invoked at every rollout step}
    \State $\{p_k\} \gets$ status file (fallback $p_k \gets 1/K$)
    \State $u_k \sim \mathrm{Uniform}(-\eta, \eta);\;\; \tilde{p}_k \gets p_k(1{+}u_k) \big/ \sum_j p_j(1{+}u_j)$ \Comment{jitter + renormalize}
    \State Assign $n_k \gets \lfloor B \tilde{p}_k \rfloor$; distribute the remaining $B - \sum_k n_k$ samples to the domains with the largest fractional parts so $\sum_k n_k = B$
    \State \Return $n_k$ prompts sampled from $\mathcal{D}_k$ for each $k$
\EndFunction
\end{algorithmic}
\end{algorithm}

\vspace{-0.5em}
\section{Experiments}
\label{sec:experiments}

\vspace{-0.5em}
\subsection{Training Setup}
\label{ssec:training_details}

We validate D$^3$-MOPD using Qwen3.6-35B-A3B~\citep{qwen36_35b_a3b} as the student model and $K{=}4$ domain-expert teachers that share the same backbone but are individually fine-tuned via GRPO~\citep{shao2024deepseekmath} on a single domain: Math, Code, Instruction Following, or Tool-use, extending the three-domain analysis of \S\ref{sec:motivation} with a fourth domain. We optimize the student using the vanilla MOPD~\citep{ma2026mopd} per-token reverse-KL (r-KL) objective on the slime~\citep{slime_github} framework, which employs an asynchronous rollout-update loop that overlaps teacher prefill with student generation. The composite signal watcher operates as an off-process job alongside the trainer following \S\ref{ssec:architecture}. It recomputes the mixture $\{p_k\}$ every $n=10$ rollout steps and delivers it to the stratified data source via a status file. All main runs use a rollout batch size of $128$ over $256$ rollout steps. Other training details are provided in Appendix~\ref{app:training_hparams}.

\paragraph{Training data.}
We construct a prompt set of ${\sim}4$k prompts per domain, formatted as single-turn message lists in JSONL with per-sample domain tags. The Math and Code splits come from DAPO-Math-17K~\citep{yu2026dapo} and CodeI/O~\citep{li2025codeio}, respectively. The Instruction Following split is filtered from Nemotron-Cascade~2~\citep{yang2026nemotron}, and the Tool-use split uses the Fission-GRPO~\citep{zhang2026robust} training set.

\subsection{Evaluation Setup}
\label{ssec:eval_setup}

We evaluate on seven benchmarks covering the four training domains: AIME 2025~\citep{maa2025aime_i} (avg@64) and HMMT November 2025~\citep{dekoninck2026beyond} (avg@32) for mathematics; LiveCodeBench Code Generation~\citep{jain2025livecodebench} (avg@6) and OJBench C++~\citep{wang2025ojbench} (default protocol) for code; IFBench~\citep{pyatkin2026generalizing} and IFEval~\citep{zhou2023instruction} for instruction following; and BFCL v3 Multi-Turn (Base)~\citep{patil2025berkeley} for tool-use. We generate responses in non-thinking mode with temperature $0.7$, top-$p$ $0.8$, top-$k$ $20$, and presence penalty $1.5$, following the official Qwen3.6-35B-A3B recommendations. 

We evaluate every $16$ rollout steps starting from step~$15$. D$^3$-MOPD and vanilla MOPD are trained under identical settings except for the mixture policy: fixed $\{p_k\}$ vs.\ online-updated $\{p_k\}$. Since raw student-to-teacher performance gaps differ across domains, averaging standard accuracies over-weights domains with larger margins. We therefore report a per-benchmark normalized score $\hat{s}_b = (s_b - s_b^{\mathrm{stu}}) / (s_b^{\mathrm{tea}} - s_b^{\mathrm{stu}})$, which scales the initial student to $0$ and the domain-expert teacher to $1$, and compute its uniform average across benchmarks. Values above $1.0$ indicate improvement beyond the respective specialist teacher.

\subsection{Main Results}
\label{ssec:main_results}

\begin{figure}[t]
  \centering
  \includegraphics[width=\linewidth]{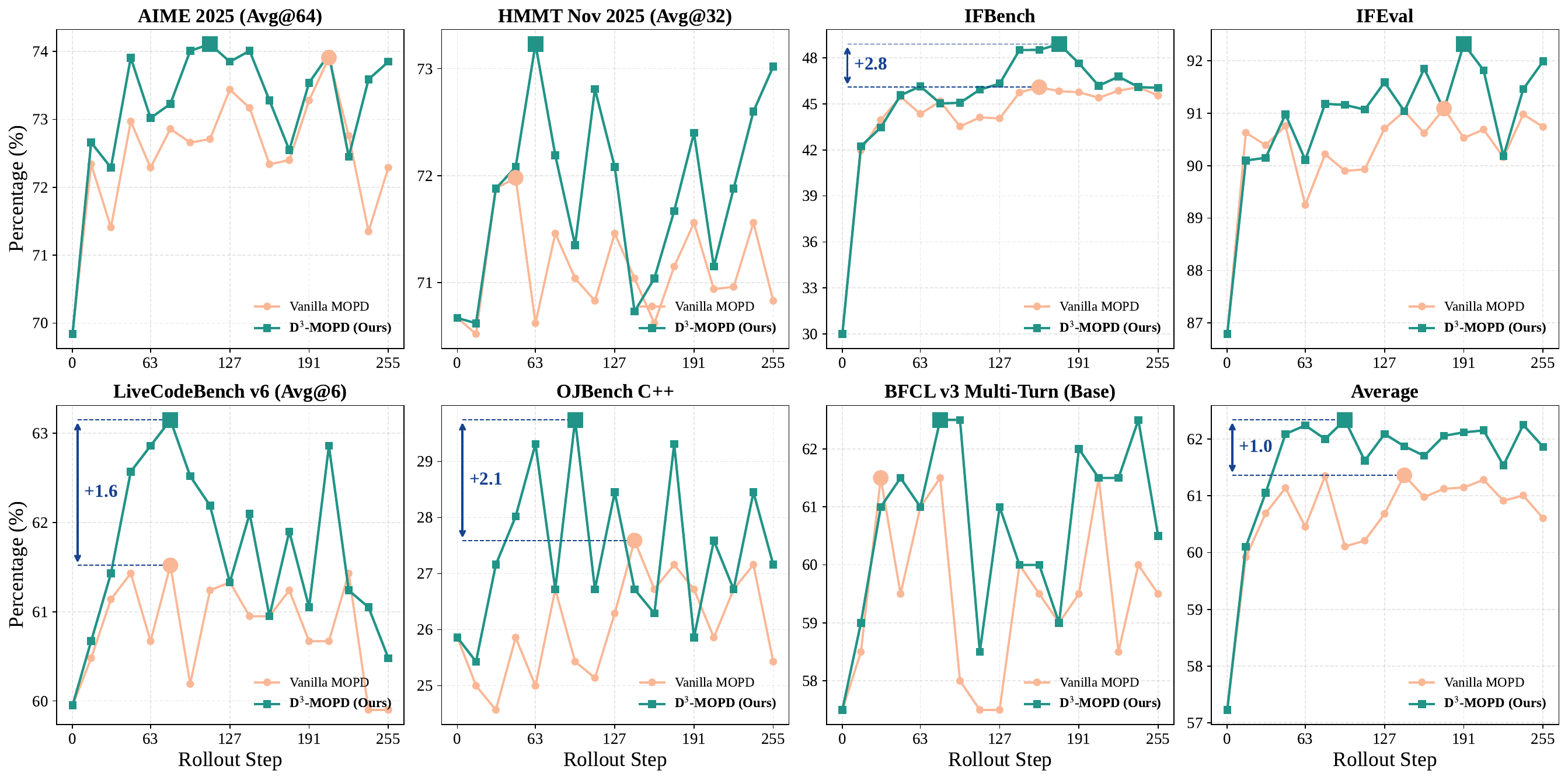}
  \vspace{-20pt}
  \caption{Per-benchmark accuracy of vanilla MOPD and D$^3$-MOPD across the 16 evaluated rollout checkpoints, spanning seven benchmarks plus their unweighted average. D$^3$-MOPD matches or outperforms the vanilla baseline on every benchmark at almost every checkpoint.}
  \label{fig:main_comparison}
  \vspace{-10pt}
\end{figure}

Figure~\ref{fig:main_comparison} shows the per-benchmark accuracy of D$^3$-MOPD and vanilla MOPD throughout training. Both methods improve rapidly during the first $\sim$50 steps before plateauing, with D$^3$-MOPD consistently maintaining a higher overall trajectory. The composite signal first activates at step~$20$ with a single-window velocity estimate and uses the full $R=3$ velocity estimate from step~$40$ onward. Consequently, the initial evaluation at step~$15$ shows minimal performance differences. Overall, D$^3$-MOPD outperforms vanilla MOPD in both peak performance and training efficiency. We highlight three findings:

\begin{itemize}[leftmargin=*,itemsep=2pt,topsep=4pt]
\item \textbf{D$^3$-MOPD peaks higher on all benchmarks.} The vanilla baseline reaches its optimal average of $61.4$ at step~$143$. In contrast, D$^3$-MOPD attains a higher peak of $62.3$ as early as step~$95$, demonstrating individual benchmark gains ranging from $+0.2$ on AIME~2025 ($74.1$ vs $73.9$) to $+2.8$ on IFBench ($48.9$ vs $46.1$). The early peak in average accuracy occurs because code benchmarks (LiveCodeBench and OJBench) converge and begin degrading after $\sim$80 steps, consistent with the domain over-optimization observed in \S\ref{sec:motivation} (Figure~\ref{fig:mopd_fixed}). By reducing the share of these converging domains, the scheduler redirects the rollout budget toward domains with active growth, raising the overall peak accuracy.

\item \textbf{D$^3$-MOPD reaches the vanilla peak with $3\times$ fewer steps.} While vanilla MOPD requires $143$ steps to reach its average score of $61.4$, D$^3$-MOPD surpasses this threshold by step~$47$ (scoring $62.1$). Furthermore, vanilla MOPD fails to establish its per-benchmark peaks until as late as step~$207$ (AIME~2025). Notably, D$^3$-MOPD matches or exceeds every one of these individual peaks within the first $79$ steps. Because the scheduler continuously shifts prompt sampling from converging domains to those still improving, each rollout step yields a larger marginal gain, accelerating overall convergence.

\item \textbf{Domain convergence order aligns with r-KL.} Code benchmarks peak earliest (LiveCodeBench by step~$79$, OJBench by step~$95$), whereas instruction-following benchmarks peak latest (IFBench at step~$175$, IFEval at step~$191$), while mathematics falls in between. This progression aligns with the r-KL convergence pattern in Figure~\ref{fig:mopd_fixed}, where code r-KL plateaus first and IF r-KL last. This consistent alignment validates r-KL as an effective monitoring signal: a domain with a plateaued r-KL has likely approached its peak accuracy, and further budget allocation risks computational waste and model degradation.
\end{itemize}

\subsection{Ablation Studies}
\label{ssec:ablation}

\begin{table}[t]
  \centering
  \small
  \setlength{\tabcolsep}{0.9pt}
  \caption{Ablation results. Best-S denotes the rollout step whose checkpoint achieves the highest average across the evaluated benchmarks over the full 256-step run. All variants differ from D$^3$-MOPD in only the ablated variable; other settings are identical.}
  \vspace{4pt}
  \label{tab:ablation}
  \begin{tabular}{lcccccccc|c}
    \toprule
    & & \multicolumn{2}{c}{Math} & \multicolumn{2}{c}{IF} & \multicolumn{2}{c}{Code} & Tool & \\
    \cmidrule(lr){3-4} \cmidrule(lr){5-6} \cmidrule(lr){7-8} \cmidrule(lr){9-9}
    Method & Best-S & AIME25 & HMMT(N) & IFBench & IFEval & LCB v6 & OJB C++ & BFCL(B) & Avg \\
    \midrule
    Vanilla MOPD          & 143 & \underline{73.17} & 71.04 & 45.74 & 91.04 & 60.95 & 27.59 & 60.00 & 61.36 \\
    \midrule
    w/o velocity            & 127 & 72.66 & 71.15 & 45.17 & 90.91 & 59.52 & \underline{28.45} & 62.00 & 61.41 {\scriptsize \color[HTML]{479A5F}($+$0.05)} \\
    w/o gap                 & 255 & 71.72 & \textbf{71.46} & \textbf{49.01} & 90.83 & 60.10 & 24.57 & 62.50 & 61.46 {\scriptsize \color[HTML]{479A5F}($+$0.10)} \\
    w/o jitter ($\eta{=}0$) &  95 & 72.66 & 71.15 & 46.58 & 89.90 & 60.95 & 27.16 & \underline{63.00} & 61.63 {\scriptsize \color[HTML]{479A5F}($+$0.27)} \\
    w/o smoothing ($R{=}1$) &  95 & 73.07 & \textbf{71.46} & \underline{46.94} & \textbf{92.14} & \underline{61.81} & 26.29 & \textbf{63.50} & \underline{62.17} {\scriptsize \color[HTML]{479A5F}($+$0.81)} \\
    \midrule
    \makecell[l]{\textbf{D$^3$-MOPD} \\ ($\eta{=}0.30$, $R{=}3$)} &  95 & \textbf{74.01} & \underline{71.35} & 45.07 & \underline{91.16} & \textbf{62.52} & \textbf{29.74} & 62.50 & \textbf{62.34} {\scriptsize \color[HTML]{479A5F}($+$0.98)} \\
    \bottomrule
  \end{tabular}
  \vspace{-10pt}
\end{table}

Table~\ref{tab:ablation} evaluates the individual contributions of three design choices in D$^3$-MOPD. The per-domain ratio trajectories of each variant are provided in Appendix~\ref{app:ablation_ratio}.

\paragraph{Composite signal vs.\ single-signal variants.}
Although both single-signal variants outperform vanilla MOPD (w/o velocity $61.41$, w/o gap $61.46$) and validate r-KL-driven dynamic scheduling, neither matches the composite formulation ($62.34$). The two signals differ in convergence speed. The w/o velocity variant peaks early at step~$127$ because it uses the current normalized r-KL levels to identify plateauing domains. In contrast, the w/o gap variant peaks at step~$255$. Since all domains' KL drops rapidly during early training, the velocity signal cannot differentiate them until descent rates diverge. Notably, the w/o gap variant records the highest IFBench score ($49.01$), consistent with IF converging last (\S\ref{ssec:main_results}). A velocity-driven scheduler allocates budget to IF since its descent rate remains high, enabling prolonged capability growth.

\paragraph{Effectiveness of batch jitter.}
Removing jitter ($\eta{=}0$) does not alter the peak step (both at $95$) but reduces the peak average by $0.71$ ($61.63$ vs $62.34$), with the most severe drops on code benchmarks (OJBench $-2.6$, LiveCodeBench $-1.6$). We hypothesize that the proportion variation from jitter acts as a form of exploration: for domains nearing their r-KL plateau, a fixed sampling ratio yields only weak gradients, and the randomized bursts from jitter may deliver stronger updates. This would explain why the rapidly plateauing code benchmarks benefit the most. The identical peak step confirms that jitter improves overall training quality without altering convergence speed.

\paragraph{Effectiveness of velocity smoothing.}
The hyperparameter $R$ controls how many non-overlapping observation windows are averaged to estimate descent velocity: $R{=}1$ relies on a single window (reactive but noisy), whereas D$^3$-MOPD adopts $R{=}3$ for smoother estimates. Both variants peak at step~$95$, and $R{=}1$ already achieves $62.17$ ($+0.81$ over vanilla), indicating that even a single-window velocity captures useful convergence information. $R{=}3$ further improves the average to $62.34$ by mitigating random r-KL variance. This smoothing prevents a single noise-induced r-KL fluctuation from temporarily zeroing the velocity signal and cutting off budget to actively learning domains. Interestingly, $R{=}1$ surpasses D$^3$-MOPD on BFCL ($63.50$ vs $62.50$), likely because tool-use exhibits high cross-checkpoint variance, where a noisier estimate encourages beneficial exploration.

\subsection{Supplementary Analysis}
\label{ssec:analysis}

\begin{table}[t]
  \vspace{-10pt} 
  \centering
  \small
  \setlength{\tabcolsep}{1pt}
  \caption{Per-benchmark peak accuracy (\%) and normalized score (\S\ref{ssec:eval_setup}) over $256$ rollout steps. Each value represents the maximum across all $16$ checkpoints. $\Delta$ rows show gaps to the domain-expert teacher. Bold marks the column best, underline the runner-up.}
  \vspace{4pt}
  \label{tab:peak_scores}
  \begin{tabular}{lccccccc|c}
    \toprule
    & \multicolumn{2}{c}{Math} & \multicolumn{2}{c}{Instruction Following} & \multicolumn{2}{c}{Code} & Tool-use & \\
    \cmidrule(lr){2-3} \cmidrule(lr){4-5} \cmidrule(lr){6-7} \cmidrule(lr){8-8}
    Method & AIME25 & HMMT(N) & IFBench & IFEval & LCB-v6 & OJB-C++ & BFCL(B) & ~Norm. \\
    \midrule
    Teacher & \textbf{76.2} & \underline{72.3} & \textbf{52.1} & \underline{91.9} & \textbf{64.8} & \underline{28.5} & \textbf{67.0} & \textbf{1.00} \\
    \midrule
    Student & 69.8 & 70.7 & 30.0 & 86.8 & 60.0 & 25.9 & 57.5 & 0.00 \\
    {\scriptsize $\Delta$(Student $-$ Teacher)} & {\scriptsize $-$6.4} & {\scriptsize $-$1.6} & {\scriptsize $-$22.1} & {\scriptsize $-$5.1} & {\scriptsize $-$4.8} & {\scriptsize $-$2.6} & {\scriptsize $-$9.5} & {\scriptsize $-$1.00} \\
    \midrule
    Vanilla MOPD & 73.9 & 72.0 & 46.1 & 91.1 & 61.5 & 27.6 & 61.5 & 0.63 \\
    {\scriptsize $\Delta$(Vanilla $-$ Teacher)} & {\scriptsize $-$2.3} & {\scriptsize $-$0.3} & {\scriptsize $-$6.0} & {\scriptsize $-$0.8} & {\scriptsize $-$3.3} & {\scriptsize $-$0.9} & {\scriptsize $-$5.5} & {\scriptsize $-$0.37} \\
    \midrule
    \textbf{D$^3$-MOPD (Ours)} & \underline{74.1} & \textbf{73.2} & \underline{48.9} & \textbf{92.3} & \underline{63.2} & \textbf{29.7} & \underline{62.5} & \underline{0.97} \\
    {\scriptsize $\Delta$(D$^3$-MOPD $-$ Teacher)} & {\scriptsize $-$2.1} & {\scriptsize \color[HTML]{479A5F}$+$0.9} & {\scriptsize $-$3.2} & {\scriptsize \color[HTML]{479A5F}$+$0.4} & {\scriptsize $-$1.6} & {\scriptsize \color[HTML]{479A5F}$+$1.2} & {\scriptsize $-$4.5} & {\scriptsize $-$0.03} \\
    \bottomrule
  \end{tabular}
  \vspace{-10pt} 
\end{table}

\paragraph{D$^3$-MOPD closes 97\% of the student-teacher gap.}
Table~\ref{tab:peak_scores} compares peak accuracy against the student baseline and domain-expert teachers. Vanilla MOPD closes $63\%$ of the average student-to-teacher gap ($\bar{\hat{s}}=0.63$), whereas D$^3$-MOPD closes $97\%$ of this gap ($\bar{\hat{s}}=0.97$). Consistent with findings by \citet{ma2026mopd}, both methods generally remain below the maximum performance of the domain-expert teachers, indicating that distillation primarily closes existing capability gaps rather than generating novel skills.
Notably, D$^3$-MOPD surpasses the teacher on three benchmarks (HMMT(N), IFEval, and OJBench~C++, $\hat{s}_b > 1.0$), whereas vanilla MOPD remains below the teacher on all tasks. Because these three benchmarks exhibit small initial student-to-teacher absolute accuracy gaps ($\Delta \le 5.1$), even modest absolute improvements translate into normalized scores exceeding $1.0$. Appendix~\ref{app:best_avg_ckpt} further compares the two methods at their respective best-average checkpoints, confirming that D$^3$-MOPD's improvement persists under this setting.

\begin{wrapfigure}{r}{0.5\linewidth}
  \centering
  \vspace{-15pt}
  \includegraphics[width=\linewidth]{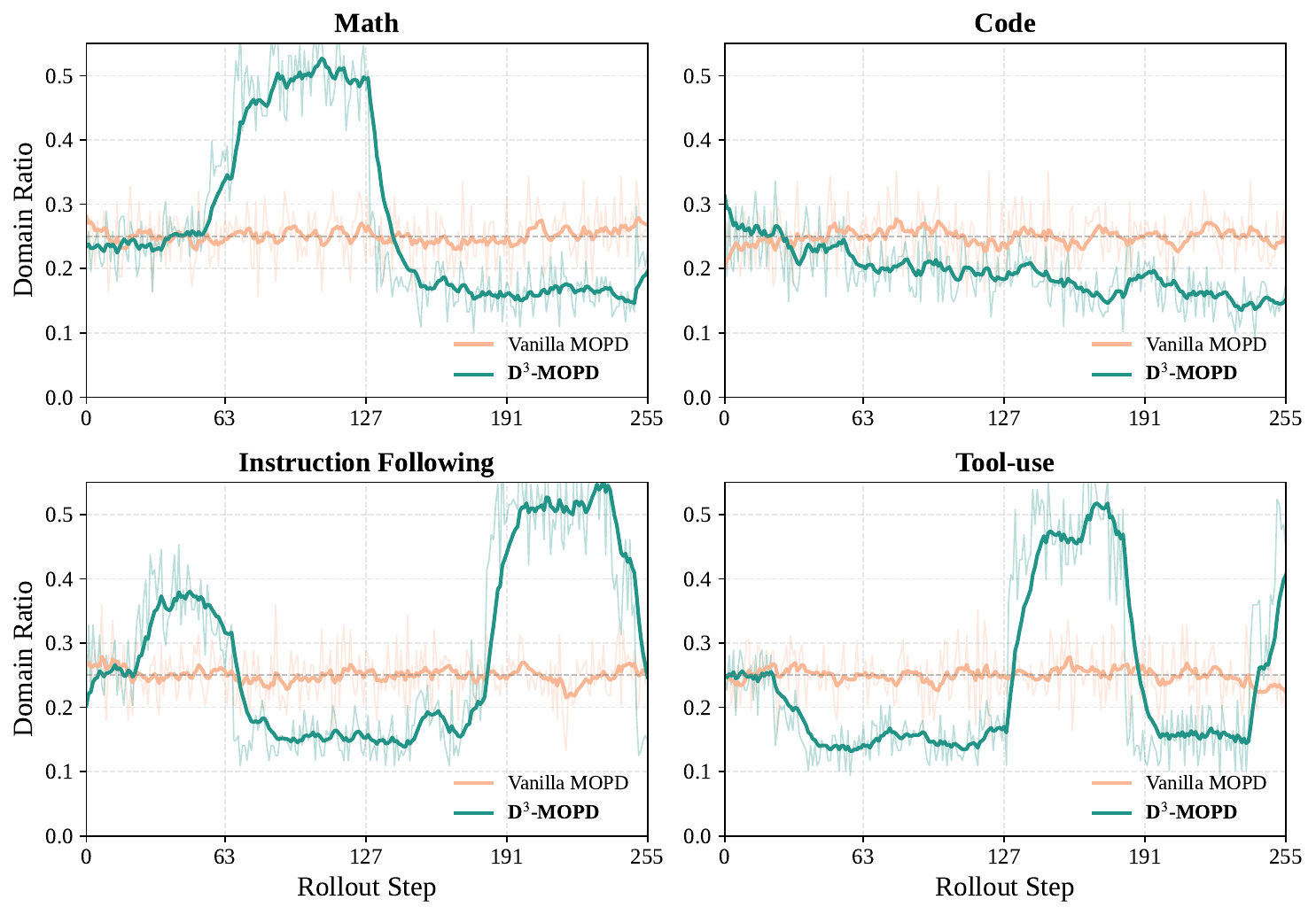}
  \vspace{-15pt}
  \caption{Per-domain sampling ratio of vanilla MOPD and D$^3$-MOPD. D$^3$-MOPD dynamically reallocates budget as domains converge.}
  \label{fig:ratio}
  \vspace{-7pt}
\end{wrapfigure}

\paragraph{D$^3$-MOPD shifts budget from converging to active domains.}
Figure~\ref{fig:ratio} shows the per-domain sampling ratio throughout training. While vanilla MOPD maintains all four domains near $0.25$ with only batch-level noise, D$^3$-MOPD dynamically adapts this allocation, closely tracking domain convergence. Code converges fastest (its r-KL plateaus first in Figure~\ref{fig:mopd_fixed}) and is steadily downsampled from $0.25$ to ${\sim}0.15$ throughout training. The freed budget flows initially to Math, whose ratio rises to ${\sim}0.50$ during steps 60--127 as the composite signal identifies Math as the domain with the largest remaining gap and active descent. Once Math converges after step~$127$, its ratio recedes, and the scheduler redirects budget toward Instruction Following and Tool-use, pushing their ratios to ${\sim}0.55$ and ${\sim}0.50$ after step~$150$. 
Because these two domains exhibit the slowest r-KL convergence, this late-stage budget increase enables the student to sustain progress where capacity for improvement remains high. Ultimately, this dynamic reallocation functions as an implicit curriculum: each domain receives concentrated training during the window when it can still improve, rather than having the budget split uniformly regardless of learning state.
\vspace{-0.5em}
\section{Related Work}
\label{sec:related}

\subsection{On-Policy Distillation}
On-policy distillation (OPD) trains a student on its own generations, supervised by the teacher's distribution under a reverse Kullback-Leibler objective~\citep{gu2024minillm,agarwal2024policy}. Recent work refines this supervision at the token or sample level: \citet{jin2026entropy} switch between forward and reverse KL based on the teacher's entropy, \citet{li2026filter} filter and reweight tokens to suppress noisy gradients, and \citet{hou2026uni} correct the mismatch between per-token KL and sequence-level outcomes. Other directions stabilize the learning target~\citep{jang-etal-2026-stable}, distill selectively from reasoning prefixes~\citep{zhang-etal-2026-fast}, or extend beyond the teacher via reward signals~\citep{yang2026learningteachergeneralizedonpolicy}. \citet{li2026rethinkingonpolicydistillationlarge} provide a systematic analysis of these design choices. These methods operate on the local supervision signal, improving which tokens or samples the student learns from. In this paper, we address a complementary dimension: instead of refining per-token supervision, we dynamically adjust the per-domain data allocation when multiple teachers are involved.

\subsection{Multi-Teacher On-Policy Distillation}
Multi-teacher OPD (MOPD) extends the single-teacher setting by distilling several domain-expert teachers into one student~\citep{ma2026mopd}. Recent large-scale post-training pipelines adopt MOPD-style distillation: MiMo-V2-Flash~\citep{coreteam2026mimov2flashtechnicalreport} and DeepSeek-V4~\citep{deepseekai2026deepseekv4highlyefficientmilliontoken} use it for capability integration, while Baichuan-M3~\citep{m3team2026baichuanm3modelingclinicalinquiry}, Nemotron-Cascade~2~\citep{yang2026nemotron}, and GLM-5~\citep{zeng2026glm} apply it in vertical-domain or staged pipelines. \citet{chen2026counteraction} further observe that domain teachers can pull the student in conflicting directions, and mitigate this through alternating per-domain training; \citet{shen2026diagnosingcalibratingtoolcallboundary} and \citet{yin2026hopdconfidenceawareheterogeneous} address failure modes such as tool-call boundary drift and heterogeneous teacher confidence. However, all existing methods keep the per-domain data ratio fixed throughout training and therefore cannot react to the asynchronous convergence documented in \S\ref{sec:motivation}. In this paper, we revisit MOPD from a scheduling perspective: the per-domain data ratio is treated as a continuous variable, updated online from the reverse-KL signal the training loop already produces.

\section{Conclusion}
\label{sec:conclusion}

This paper identifies and addresses a systematic inefficiency in multi-teacher on-policy distillation (MOPD) arising from the mismatch between its static domain mixture and the student's asynchronous per-domain convergence rates. D$^3$-MOPD resolves this by combining each domain's remaining KL gap with its descent velocity into a composite signal and mapping the result to an updated sampling ratio through a softmax-floor formulation, all without modifying the core training loop. Experiments across four domains with a Qwen3.6-35B-A3B student show that D$^3$-MOPD closes $97\%$ of the student-teacher performance gap (vs.\ $63\%$ for vanilla MOPD) while reaching the baseline's peak accuracy in $3\times$ fewer rollout steps. Ablation studies confirm that each design choice, namely the composite signal, batch jitter, and velocity smoothing, contributes to both peak accuracy and convergence speed. Trajectory analysis reveals an emergent implicit curriculum in which the scheduler concentrates budget on each domain during its active learning window.

\bibliographystyle{styles/colm2026_conference}
\bibliography{bibliography/references}

\appendix
\newpage
\section{Contributors}
\label{app:contributors}
\noindent
{Authors are listed in order of contribution.}

\vspace{4pt}
\begingroup
\renewcommand{\thefootnote}{\fnsymbol{footnote}}
\noindent
{
Zechen Sun\footnotemark[1],
Zhiwei Zhang\footnotemark[1],
Fei Zhao\footnotemark[2],
Juntao Li\textsuperscript{\faEnvelope[regular]}, 
Huayu Deng,
Guojian Zhan,
Wenliang Chen,
Zhaokai Luo,
Yao Hu,
Mu Chuan\textsuperscript{\faEnvelope[regular]}}

\footnotetext[1]{Equal contribution.}
\footnotetext[2]{Project lead.}

\begingroup
\renewcommand{\thefootnote}{\relax} 
\footnotetext[3]{\faEnvelope[regular]~Corresponding authors.} 
\endgroup

\endgroup

\section{Training Hyperparameters and Infrastructure}
\label{app:training_hparams}

All main and ablation runs share the hyperparameters below; only the mixture-scheduling components (\S\ref{ssec:algorithm}) differ across variants.

\subsection{Framework, Watcher, and Data Pipeline}
\label{app:watcher_details}
The slime~\citep{slime_github} framework employs an asynchronous rollout-update loop that overlaps teacher prefill with student generation. On top of it, the composite-signal watcher operates as an off-process job alongside the trainer following \S\ref{ssec:architecture}: it recomputes the mixture $\{p_k\}$ every $n=10$ rollout steps and delivers it to the stratified data source via a status file that the trainer never blocks on. Training data is formatted as single-turn message lists in JSONL with per-sample domain tags, which the stratified data source consumes to compose each batch under the target mixture.

\subsection{\texorpdfstring{D$^3$}{D3}-MOPD Scheduler}
Table~\ref{tab:watcher_hparams} lists the scheduler settings used for the D$^3$-MOPD composite-signal method.

\begin{table}[h]
\vspace{-10pt}
\centering
\small
\caption{D$^3$-MOPD scheduler hyperparameters.}
\vspace{4pt}
\label{tab:watcher_hparams}
\begin{tabular}{lll}
\toprule
Symbol & Description & Value \\
\midrule
$n$        & Watcher update cadence (rollout steps)                          & $10$ \\
$W$        & Window length in the descent-velocity estimate (rollout steps)  & $10$ \\
$R$        & Number of non-overlapping windows averaged in $v_k$              & $3$ \\
$S_0$      & Seed length for the initial-KL normalizer $\mathrm{KL}_k^{(0)}$ & $5$ \\
--         & EMA window applied to raw per-step KL                           & $10$ \\
--         & Numerical KL floor ($\epsilon_{\mathrm{KL}}$) in denominators   & $0.15$ \\
$T$        & Softmax temperature in Eq.~\ref{eq:ratio}                       & $0.5$ \\
$\epsilon$ & Per-domain mixture floor in Eq.~\ref{eq:ratio}                  & $0.10$ \\
$\eta$     & Batch-level jitter amplitude                                    & $0.30$ \\
--         & File-poll interval of the watcher process (seconds)             & $300$ \\
\bottomrule
\end{tabular}
\end{table}

\subsection{Student Training}
The student is trained with slime~\citep{slime_github} on top of the Megatron backend; rollout inference is handled by an SGLang engine group co-located with the trainer. All runs use $7$ nodes of GPUs: $4$ trainer nodes, $2$ SGLang rollout nodes, and $1$ teacher node shared by $K{=}4$ teachers, connected via RDMA InfiniBand. Table~\ref{tab:training_hparams} lists the training-loop hyperparameters, and Table~\ref{tab:mopd_perf_hparams} lists parallelism and performance settings.

\begin{table}[h]
\vspace{-10pt}
\centering
\small
\caption{Training hyperparameters.}
\vspace{4pt}
\label{tab:training_hparams}
\begin{tabular}{ll}
\toprule
Setting & Value \\
\midrule
Student initialization                             & Qwen3.6-35B-A3B \\
Number of teachers $K$                             & $4$ \\
Total rollout steps                                & $256$ \\
Rollout batch size $B$ (prompts per step)          & $128$ \\
Responses sampled per prompt                       & $4$ \\
Global mini-batch size                             & $128$ \\
Gradient-accumulation micro-steps per rollout step & $4$ \\
Optimizer                                          & Adam \\
Learning rate                                      & $1 \times 10^{-6}$ \\
Learning-rate schedule                             & constant \\
Weight decay                                       & $0.1$ \\
Adam $(\beta_1, \beta_2)$                          & $(0.9, 0.98)$ \\
Policy-loss variant (advantage estimator)          & CISPO~\citep{chen2025minimax} \\
Clip ratio (low, high)                             & $(0.2, 0.2)$ \\
OPD reverse-KL coefficient                         & $1.0$ \\
Max prompt length                                  & $16{,}384$ tokens \\
Max response length                                & $8{,}192$ tokens \\
Rollout sampling temperature                       & $1.0$ \\
Chat-template thinking mode                        & disabled \\
\bottomrule
\end{tabular}
\end{table}

\begin{table}[h]
\centering
\small
\caption{Parallelism and performance settings (Megatron backend).}
\vspace{4pt}
\label{tab:mopd_perf_hparams}
\begin{tabular}{ll}
\toprule
Setting & Value \\
\midrule
Tensor-model parallel size                          & $2$ \\
Context-parallel size                               & $4$ \\
Expert-model parallel size                          & $8$ \\
Max tokens per GPU per micro-batch                  & $6{,}144$ \\
SGLang rollout: GPUs per engine                     & $4$ \\
\bottomrule
\end{tabular}
\end{table}

\paragraph{Throughput overhead.}
Because the watcher runs as a standalone process that reads the training log and writes updated ratios to a status file without requiring the trainer to pause, it adds no synchronization overhead. Over the full $256$-step run the mean throughputs of D$^3$-MOPD ($520$ tokens/GPU/s) and vanilla MOPD ($531$ tokens/GPU/s) differ by only $2.1\%$, a gap attributable to the shift in sequence-length distribution as the scheduler redirects budget toward longer-response domains rather than to any computational overhead from the watcher.

\section{Evaluation Details}
\label{app:eval_details}

\paragraph{Sampling hyperparameters.}
We generate responses in non-thinking mode with temperature $0.7$, top-$p$ $0.8$, top-$k$ $20$, and presence penalty $1.5$, following the Qwen3.6-35B-A3B recommendations.

\paragraph{Per-benchmark metrics.}
We report avg@64 for AIME 2025, avg@32 for HMMT November 2025, and avg@6 for LiveCodeBench Code Generation; OJBench C++, IFBench, IFEval, and BFCL v3 Multi-Turn (Base) follow their respective standard protocols.

\section{Best-Average Checkpoint Comparison}
\label{app:best_avg_ckpt}

Table~\ref{tab:peak_scores} reports the per-benchmark \emph{peak} accuracy, where each benchmark's best value may come from a different checkpoint. Table~\ref{tab:best_avg_ckpt} instead reports the accuracy at the single checkpoint with the highest average score, reflecting the performance of a single deployable model.

\begin{table}[h]
  \vspace{-10pt}
  \centering
  \small
  \setlength{\tabcolsep}{1pt}
  \caption{Per-benchmark accuracy (\%) and normalized score at the best-average checkpoint (vanilla MOPD: step~$143$; D$^3$-MOPD: step~$95$). The Teacher row reports each domain's dedicated expert model (one specialized model per domain, not a single model). Unlike Table~\ref{tab:peak_scores}, all values in each row come from a single checkpoint. Bold marks the column best, underline the runner-up.}
  \vspace{4pt}
  \label{tab:best_avg_ckpt}
  \begin{tabular}{lccccccc|c}
    \toprule
    & \multicolumn{2}{c}{Math} & \multicolumn{2}{c}{Instruction Following} & \multicolumn{2}{c}{Code} & Tool-use & \\
    \cmidrule(lr){2-3} \cmidrule(lr){4-5} \cmidrule(lr){6-7} \cmidrule(lr){8-8}
    Method & AIME25 & HMMT(N) & IFBench & IFEval & LCB-v6 & OJB-C++ & BFCL(B) & Norm. \\
    \midrule
    Teacher & \textbf{76.2} & \textbf{72.3} & \textbf{52.1} & \textbf{91.9} & \textbf{64.8} & \underline{28.5} & \textbf{67.0} & \textbf{1.00} \\
    \midrule
    Student & 69.8 & 70.7 & 30.0 & 86.8 & 60.0 & 25.9 & 57.5 & 0.00 \\
    \midrule
    Vanilla MOPD {\scriptsize (step 143)} & 73.2 & 71.0 & \underline{45.7} & 91.0 & 61.0 & 27.6 & 60.0 & 0.48 \\
    {\scriptsize $\Delta(\text{Vanilla} - \text{Teacher})$} & {\scriptsize $-$3.0} & {\scriptsize $-$1.3} & {\scriptsize $-$6.4} & {\scriptsize $-$0.9} & {\scriptsize $-$3.8} & {\scriptsize $-$0.9} & {\scriptsize $-$7.0} & {\scriptsize $-$0.52} \\
    \midrule
    \textbf{D$^3$-MOPD} {\scriptsize (step 95)} & \underline{74.0} & \underline{71.4} & 45.1 & \underline{91.2} & \underline{62.5} & \textbf{29.7} & \underline{62.5} & \underline{0.73} \\
    {\scriptsize $\Delta(\text{D}^3\text{-MOPD} - \text{Teacher})$} & {\scriptsize $-$2.2} & {\scriptsize $-$0.9} & {\scriptsize $-$7.0} & {\scriptsize $-$0.7} & {\scriptsize $-$2.3} & {\scriptsize \color[HTML]{479A5F}$+$1.2} & {\scriptsize $-$4.5} & {\scriptsize $-$0.27} \\
    {\scriptsize $\Delta(\text{D}^3\text{-MOPD} - \text{Vanilla})$} & {\scriptsize \color[HTML]{479A5F}$+$0.8} & {\scriptsize \color[HTML]{479A5F}$+$0.4} & {\scriptsize $-$0.6} & {\scriptsize \color[HTML]{479A5F}$+$0.2} & {\scriptsize \color[HTML]{479A5F}$+$1.5} & {\scriptsize \color[HTML]{479A5F}$+$2.1} & {\scriptsize \color[HTML]{479A5F}$+$2.5} & {\scriptsize \color[HTML]{479A5F}$+$0.25} \\
    \bottomrule
  \end{tabular}
\end{table}

\paragraph{Comparison with per-benchmark peaks.}
Compared with Table~\ref{tab:peak_scores}, the normalized scores drop for both methods (vanilla MOPD: $0.63 \to 0.48$; D$^3$-MOPD: $0.97 \to 0.73$), since no single checkpoint can maximize every benchmark at once. The drop is concentrated in benchmarks whose peaks occur far from the best-average step: for example, D$^3$-MOPD's HMMT(N) falls from $73.2$ (per-benchmark peak) to $71.4$ at step~$95$, indicating that HMMT(N) peaks at a later checkpoint. Nevertheless, D$^3$-MOPD at step~$95$ still outperforms vanilla MOPD at step~$143$ on six of seven benchmarks and achieves a notably higher normalized score ($0.73$ vs.\ $0.48$), confirming that the improvement holds when deploying a single checkpoint.

\section{Generalization to 4B Student}
\label{app:4b_exp}

To verify that D$^3$-MOPD generalizes across model scales, we replicate the training pipeline with Qwen3.5-4B~\citep{qwen3.5} as the student. The four domain-specific teachers are obtained by GRPO-training the same model on the corresponding domain data used for the Qwen3.6-35B-A3B experiments. Training hyperparameters (Table~\ref{tab:training_hparams}) and scheduler settings (Table~\ref{tab:watcher_hparams}) follow the same configuration with minor scale-related adjustments.

\begin{table}[h]
  \vspace{-10pt}
  \centering
  \small
  \setlength{\tabcolsep}{1pt}
  \caption{Per-benchmark accuracy (\%) and normalized score with Qwen3.5-4B student (vanilla MOPD: step~$159$; D$^3$-MOPD: step~$119$). The Teacher row reports each domain's dedicated GRPO-trained 4B expert. Bold marks the column best, underline the runner-up.}
  \vspace{4pt}
  \label{tab:4b_results}
  \begin{tabular}{lccccccc|c}
    \toprule
    & \multicolumn{2}{c}{Math} & \multicolumn{2}{c}{Instruction Following} & \multicolumn{2}{c}{Code} & Tool-use & \\
    \cmidrule(lr){2-3} \cmidrule(lr){4-5} \cmidrule(lr){6-7} \cmidrule(lr){8-8}
    Method & AIME25 & HMMT(N) & IFBench & IFEval & LCB-v6 & OJB-C++ & BFCL(B) & Norm. \\
    \midrule
    Teacher & \textbf{63.4} & \textbf{66.4} & \textbf{55.9} & \textbf{85.3} & 53.3 & 18.1 & \underline{63.0} & \underline{1.00} \\
    \midrule
    Student & 47.8 & 51.7 & 35.9 & 85.3 & 39.4 & 13.8 & 51.0 & 0.00 \\
    \midrule
    Vanilla MOPD {\scriptsize (step 159)} & 59.3 & 63.4 & 45.4 & 82.7 & \underline{53.7} & \underline{18.5} & \underline{63.0} & 0.86 \\
    {\scriptsize $\Delta(\text{Vanilla} - \text{Teacher})$} & {\scriptsize $-$4.1} & {\scriptsize $-$3.0} & {\scriptsize $-$10.5} & {\scriptsize $-$2.6} & {\scriptsize \color[HTML]{479A5F}$+$0.4} & {\scriptsize \color[HTML]{479A5F}$+$0.4} & {\scriptsize 0.0} & {\scriptsize $-$0.14} \\
    \midrule
    \textbf{D$^3$-MOPD} {\scriptsize (step 119)} & \underline{61.8} & \underline{64.4} & \underline{49.5} & \underline{84.5} & \textbf{54.6} & \textbf{19.8} & \textbf{64.5} & \textbf{1.01} \\
    {\scriptsize $\Delta(\text{D}^3\text{-MOPD} - \text{Teacher})$} & {\scriptsize $-$1.6} & {\scriptsize $-$2.0} & {\scriptsize $-$6.4} & {\scriptsize $-$0.8} & {\scriptsize \color[HTML]{479A5F}$+$1.3} & {\scriptsize \color[HTML]{479A5F}$+$1.7} & {\scriptsize \color[HTML]{479A5F}$+$1.5} & {\scriptsize \color[HTML]{479A5F}$+$0.01} \\
    {\scriptsize $\Delta(\text{D}^3\text{-MOPD} - \text{Vanilla})$} & {\scriptsize \color[HTML]{479A5F}$+$2.5} & {\scriptsize \color[HTML]{479A5F}$+$1.0} & {\scriptsize \color[HTML]{479A5F}$+$4.1} & {\scriptsize \color[HTML]{479A5F}$+$1.8} & {\scriptsize \color[HTML]{479A5F}$+$0.9} & {\scriptsize \color[HTML]{479A5F}$+$1.3} & {\scriptsize \color[HTML]{479A5F}$+$1.5} & {\scriptsize \color[HTML]{479A5F}$+$0.15} \\
    \bottomrule
  \end{tabular}
\end{table}

\paragraph{Results.}
Table~\ref{tab:4b_results} shows that the improvement pattern observed at the 35B-A3B scale carries over to the 4B student. D$^3$-MOPD outperforms vanilla MOPD on all seven benchmarks, with per-benchmark gains ranging from $+0.9$ (LCB-v6) to $+4.1$ (IFBench). The normalized score of D$^3$-MOPD ($1.01$) slightly surpasses the composite teacher ceiling ($1.00$), whereas vanilla MOPD reaches $0.86$. Both methods surpass the domain-expert teachers on the code benchmarks (LCB-v6 and OJB-C++), and D$^3$-MOPD additionally exceeds the tool-use teacher on BFCL(B) by $+1.5$. The remaining gap to the teacher is largest on IFBench ($-6.4$), consistent with instruction following being the hardest domain for distillation at both model scales.

\paragraph{Convergence speed.}
D$^3$-MOPD reaches its best-average checkpoint $40$ steps earlier than vanilla MOPD (step~$119$ vs.\ step~$159$), a $25\%$ reduction in the number of training steps to achieve peak average performance. This faster convergence aligns with the 35B-A3B result (step~$95$ vs.\ step~$143$, Table~\ref{tab:best_avg_ckpt}), indicating that dynamic domain scheduling accelerates convergence across model scales.

\section{Theoretical Analysis of Composite-Signal Allocation}
\label{app:theory}

We provide a simplified analysis showing that under an exponential decay model, the composite signal $s_k = \widetilde{\mathrm{KL}}_k \cdot v_k$ is proportional to the marginal reduction in the normalized gap, justifying its use as a greedy-optimal allocation criterion.

\paragraph{Exponential decay model.}
We model each domain's reverse-KL as decaying exponentially in its cumulative training exposure:
\begin{equation}
    \mathrm{KL}_k(n_k) = \mathrm{KL}_k(0) \cdot e^{-\lambda_k n_k},
    \label{eq:exp_decay}
\end{equation}
where $n_k$ is the cumulative sample count allocated to domain $k$ and $\lambda_k > 0$ is a domain-specific learning rate. This captures the empirical observation (\S\ref{ssec:dynamics}) that per-domain KL decreases roughly exponentially at domain-specific rates.

\begin{proposition}
\label{prop:greedy}
Under the exponential decay model (Eq.~\ref{eq:exp_decay}), with batch size $B$, window length $W$, uniform allocation $p_k = 1/K$, and $\lambda_k B\,W/K \ll 1$, the composite signal satisfies
\begin{equation}
    s_k \;=\; \widetilde{\mathrm{KL}}_k \cdot v_k \;\approx\; \frac{B\,W}{K} \cdot \left|\frac{\partial \widetilde{\mathrm{KL}}_k}{\partial n_k}\right|.
\end{equation}
Allocating samples to $\argmax_k s_k$ thus maximizes the instantaneous reduction in the average normalized gap $\bar{g} = \frac{1}{K}\sum_k \widetilde{\mathrm{KL}}_k$.
\end{proposition}

\begin{proof}
The normalized gap is $\widetilde{\mathrm{KL}}_k = \mathrm{KL}_k(n_k)/\mathrm{KL}_k(0) = e^{-\lambda_k n_k}$, with marginal reduction
\[
    \left|\frac{\partial \widetilde{\mathrm{KL}}_k}{\partial n_k}\right| = \lambda_k \, e^{-\lambda_k n_k} = \lambda_k \, \widetilde{\mathrm{KL}}_k.
\]
Under uniform allocation $p_k = 1/K$, each window of $W$ rollout steps delivers $B\,W/K$ samples to domain $k$. Substituting into Eq.~\ref{eq:delta_window}, the single-window relative change is $\delta_k^{(i)} = e^{-\lambda_k B\,W/K} - 1$, constant across windows $i$. The descent velocity (Eq.~\ref{eq:descent_velocity}) therefore reduces to $v_k = 1 - e^{-\lambda_k B\,W/K}$. A first-order expansion yields $v_k \approx \lambda_k B\,W/K$, giving
\[
    s_k = \widetilde{\mathrm{KL}}_k \cdot v_k \;\approx\; \frac{B\,W}{K} \cdot \lambda_k \, \widetilde{\mathrm{KL}}_k \;=\; \frac{B\,W}{K} \cdot \left|\frac{\partial \widetilde{\mathrm{KL}}_k}{\partial n_k}\right|.
\]
Since $B\,W/K$ is constant across domains, $\argmax_k s_k = \argmax_k |\partial \widetilde{\mathrm{KL}}_k / \partial n_k|$.
\end{proof}

\paragraph{Interpretation.}
Proposition~\ref{prop:greedy} shows that the composite signal ranks domains by the magnitude of their normalized marginal KL reduction. The initial-KL normalization in $\widetilde{\mathrm{KL}}_k$ (Eq.~\ref{eq:normalized_kl}) is essential: without it, domains with large absolute KL (e.g., IF, whose KL is ${\sim}65\times$ that of Math) would dominate the allocation regardless of their relative improvement rate. The first-order condition $\lambda_k B\,W/K \ll 1$ requires per-window KL changes to be small, which is satisfied in practice by the short window length ($W = 10$ steps).

\paragraph{Practical robustness.}
The softmax-floor mapping (Eq.~\ref{eq:ratio}) smooths the greedy $\argmax_k s_k$ allocation: the temperature $T$ distributes samples across domains in proportion to their exponential signal strength rather than concentrating on a single domain, while the floor $\epsilon$ guarantees every domain retains a minimum sample rate, preventing catastrophic forgetting of domains that have temporarily plateaued but may resume descent later in training.

\paragraph{Remark on the exponential decay assumption.}
The exponential decay model is a simplifying assumption that we adopt to derive the proportionality between $s_k$ and the marginal KL reduction in closed form. The method itself does not require this assumption to hold exactly, because the composite signal is self-correcting by construction: if a domain deviates from exponential decay and enters an unexpected plateau, its descent velocity $v_k$ drops toward zero and automatically reduces its allocation regardless of the underlying functional form. The theoretical analysis therefore serves as a design rationale for choosing the gap-velocity product as the allocation criterion rather than a necessary condition for the method to produce effective schedules.

\section{Composite Signal Trajectories}
\label{app:signal_dynamics}

Figure~\ref{fig:signal_dynamics} visualizes how the composite signal $\tilde{s}_k$ and the normalized KL $\widetilde{\mathrm{KL}}_k$ evolve for each domain over the full 256-step training run. Panel~(e) shows the corresponding allocation ratio $p_k$.

The scheduler concentrates budget on whichever domain currently exhibits both a large remaining gap and rapid descent, producing a sequence of distinct allocation phases. In the first phase (steps 20--55), IF receives the highest signal because its initial KL gap is the largest among all domains and it descends rapidly once training begins, causing the allocation to reach approximately $38\%$ while the other three domains receive near-floor shares. As IF's descent velocity slows around step~55, Math takes over as the dominant domain and holds $\tilde{s}_{\text{math}} = 1.0$ for roughly 60 consecutive steps, during which its normalized KL drops steadily from 0.42 to 0.25. The scheduler then redirects budget to Tool-use (steps 120--160) and finally returns to IF (steps 180--250) as IF's accumulated gap reopens after an extended period of reduced allocation.

The self-correcting property of the composite signal is visible in the interplay between the shaded signal regions and the KL curves within each panel. When a domain's signal drops to zero and it receives only the floor allocation $\epsilon = 0.10$, its normalized KL tends to plateau or rise (for example, IF's normalized KL climbs from 0.29 to 0.65 during steps 65--120), which eventually increases the remaining gap component and restores a nonzero signal in a later update. This feedback loop ensures that no domain is permanently neglected even without explicit revisitation logic in the scheduler.

\begin{figure}[h]
    \centering
    \includegraphics[width=0.55\linewidth]{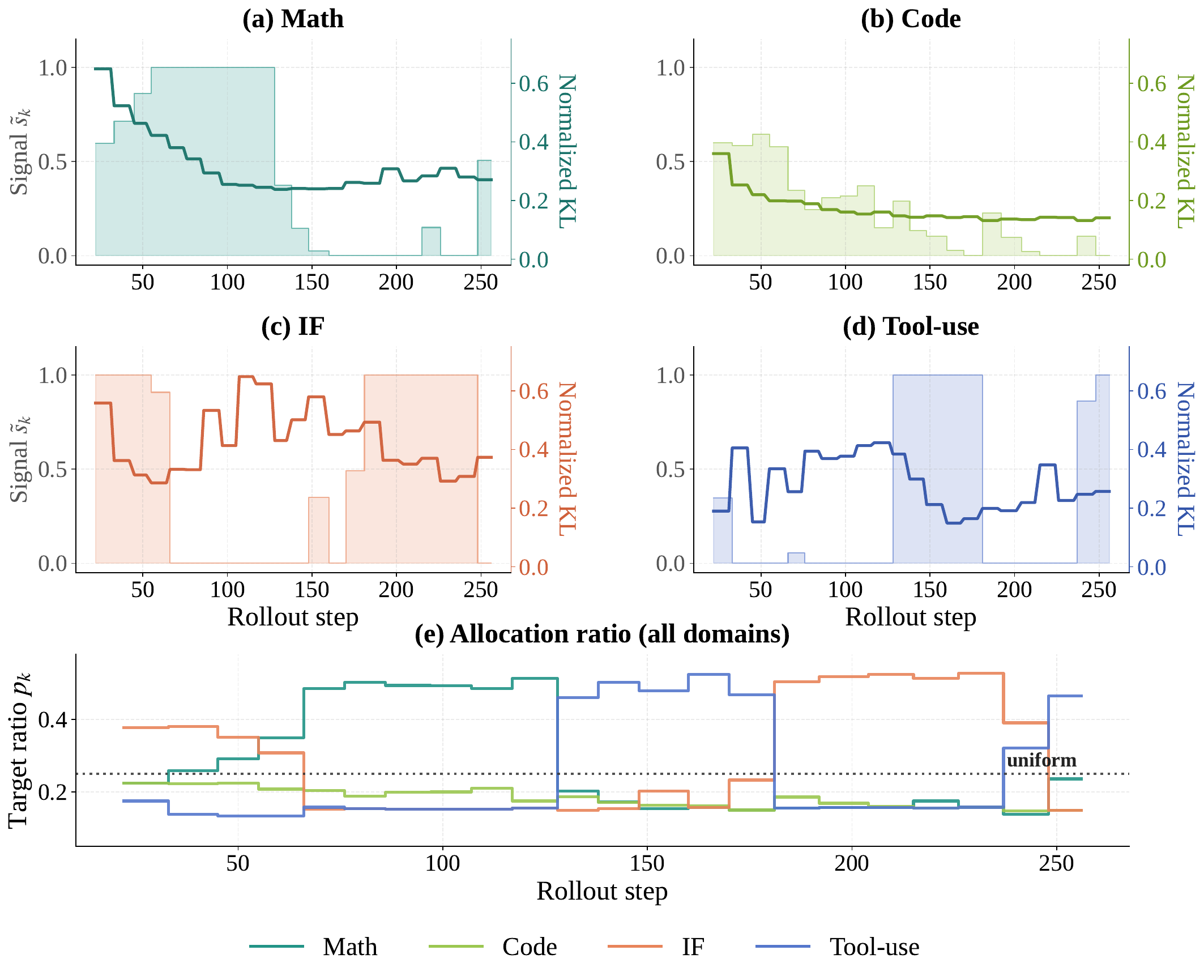}
    \vspace{-15pt}
    \caption{Per-domain composite signal and normalized KL over 256 rollout steps, with the resulting allocation ratio in panel~(e). The scheduler produces a sequence of phase transitions that track the domain with the largest remaining gap and fastest descent at each point in training.}
    \label{fig:signal_dynamics}
\end{figure}

\begin{figure}[h]
    \centering
    \includegraphics[width=\linewidth]{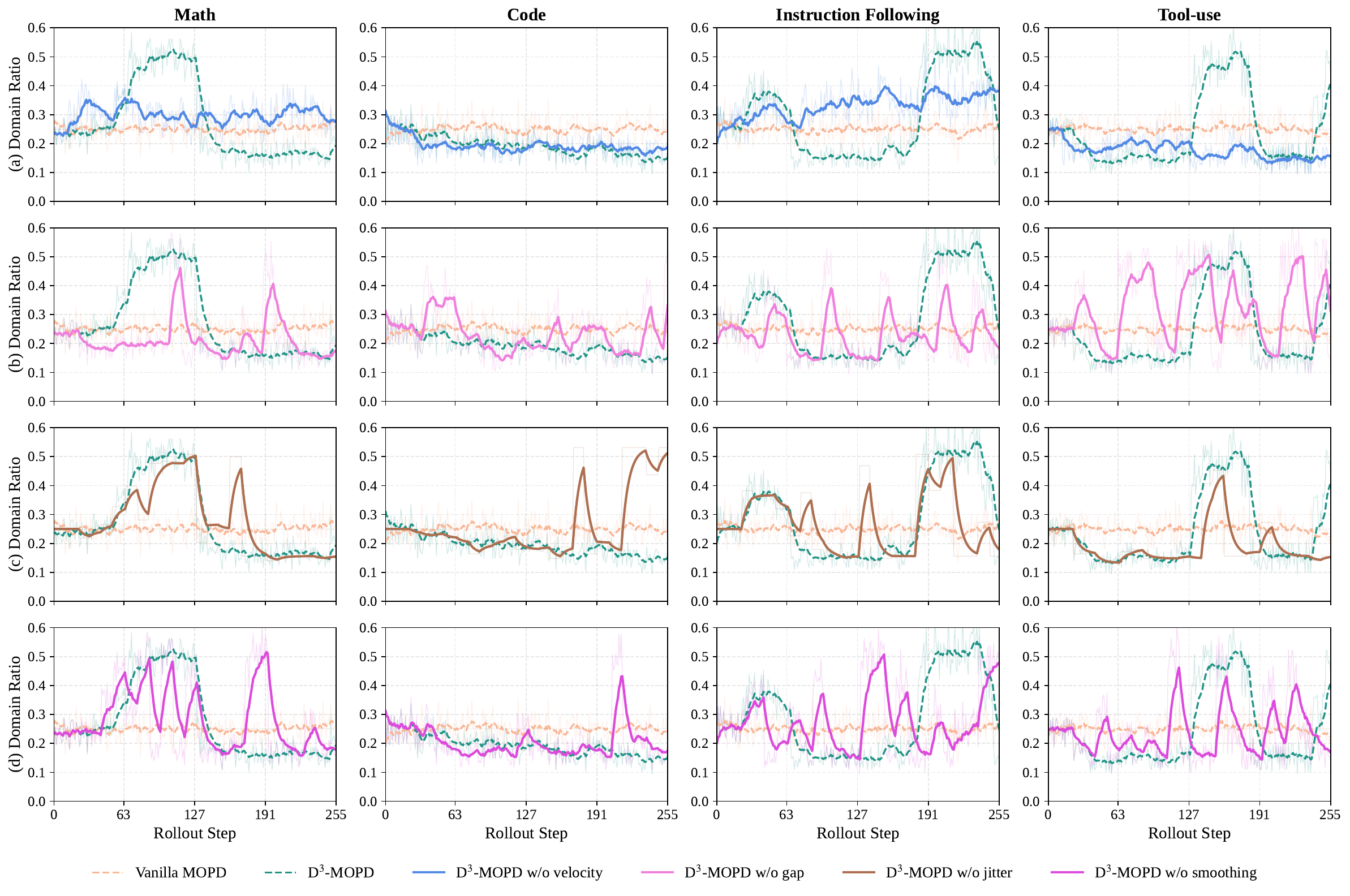}
    \vspace{-15pt}
    \caption{Per-domain sampling ratio of each ablation variant vs.\ D$^3$-MOPD and vanilla MOPD over 256 rollout steps. (a)~w/o velocity, (b)~w/o gap, (c)~w/o jitter, (d)~w/o smoothing.}
    \label{fig:abl_ratio}
  \end{figure}

\section{Ablation Domain Ratio Trajectories}
\label{app:ablation_ratio}

Figure~\ref{fig:abl_ratio} shows the per-domain sampling ratio of each ablation variant alongside vanilla MOPD and D$^3$-MOPD over 256 rollout steps.

\begin{itemize}[leftmargin=*,itemsep=2pt,topsep=4pt]
\item \textbf{w/o velocity (row a).} Without descent velocity, the scheduler relies solely on the KL gap. The ratios are flatter than D$^3$-MOPD: Math rises to ${\sim}0.30$ but never reaches the ${\sim}0.50$ peak of D$^3$-MOPD, and the late-stage shift toward IF and Tool-use is weaker. Because the gap signal alone cannot distinguish an actively descending domain from one plateaued at a high KL, the scheduler spreads budget uniformly.

\item \textbf{w/o gap (row b).} Without the remaining gap, the scheduler uses only descent velocity. Early in training all domains descend rapidly, so the velocity signal cannot differentiate them, producing large oscillations before step~$63$. Later, IF is identified as the fastest-descending domain and heavily upweighted, but the lack of gap information causes abrupt ratio swings (e.g., Math drops to ${\sim}0.10$ around step~$100$ before rebounding). This instability explains why the w/o gap variant peaks latest at step~$255$ (Table~\ref{tab:ablation}).

\item \textbf{w/o jitter (row c).} The overall trajectory shape is similar to D$^3$-MOPD, confirming that jitter does not alter the scheduler's long-run allocation. However, the ratio curves show sharper step-like transitions between watcher updates, most visible in Code near step~$200$ and Tool-use near step~$240$. The corresponding accuracy drops on code benchmarks (Table~\ref{tab:ablation}) suggest that this reduced per-batch variation hurts domains near their r-KL plateau.

\item \textbf{w/o smoothing (row d).} A single observation window ($R{=}1$) makes the velocity estimate more reactive but noisier. The ratio curves oscillate visibly more than D$^3$-MOPD ($R{=}3$), with frequent spikes in Math and Code around steps~$60$--$130$. The overall allocation pattern nonetheless remains close to D$^3$-MOPD, which explains why this variant still achieves $62.17$ ($+0.81$ in Table~\ref{tab:ablation}).

\end{itemize}

\section{Robustness Extensions}
\label{app:robustness}

We describe two optional flags that extend the composite signal (Eq.~\ref{eq:composite}) to address two potential failure modes we found in practice. Both flags are off by default, in which case the scheduler produces the same output as Algorithm~\ref{alg:d3mopd}.

\subsection{Velocity Floor for KL-Rebound Domains}
\label{app:robustness_vfloor}

\paragraph{Failure mode.} The descent velocity $v_k = \max(0, -\tfrac{1}{R}\sum_i \delta_k^{(i)})$ is clipped at zero, so any KL rise over the last $RW$ steps forces $v_k = 0$ regardless of the remaining gap $\widetilde{\mathrm{KL}}_k$. The composite $s_k = \widetilde{\mathrm{KL}}_k \cdot v_k$ then becomes zero and the allocation $p_k$ drops to the floor $\epsilon$. A domain with substantial remaining gap can therefore lose most of its budget only because its KL briefly went up.

\paragraph{Extension.} We replace the hard clip with a soft one controlled by a velocity floor $\phi \in [0, 1]$:
\begin{small}
\begin{equation}
    \tilde{v}_k(t) \;=\; \phi \;+\; (1 - \phi)\,\min\!\bigl(v_k(t),\, 1\bigr),
    \label{eq:vfloor}
\end{equation}
\end{small}
so $\tilde{v}_k \in [\phi, 1]$, which recovers Eq.~\ref{eq:descent_velocity} at $\phi = 0$. Any $\phi > 0$ gives a domain whose velocity is clipped a gap-only fallback $\tilde{s}_k = \phi \cdot \widetilde{\mathrm{KL}}_k$, keeping its budget proportional to the remaining gap.

\paragraph{Empirical validation.} We rerun the main experiment with $\phi = 0.2$ and all other settings unchanged. Figure~\ref{fig:vfloor_if} compares the IF domain under the two settings. IF exhibits a KL rebound window between steps 60 and 150 (shaded band) during which the baseline clips its velocity to zero. As a result (a) the composite signal $\tilde s_\text{IF}$ drops to $0$ and (c) the target ratio $p_\text{IF}$ drops to the floor $\epsilon = 0.10$. This happens even though (b) shows the normalised KL is still $\widetilde{\mathrm{KL}}_\text{IF} \approx 0.5$ and IF is far from converged. With $\phi = 0.2$ the composite signal stays at or above $\phi$ throughout the rebound and $p_\text{IF}$ stays between $0.30$ and $0.40$. IF's own KL rebound is also smaller as a side effect. Figure~\ref{fig:vfloor_signal_dynamics} plots the full four-domain trajectory of the same trial.

\begin{figure}[t]
    \centering
    \includegraphics[width=\linewidth]{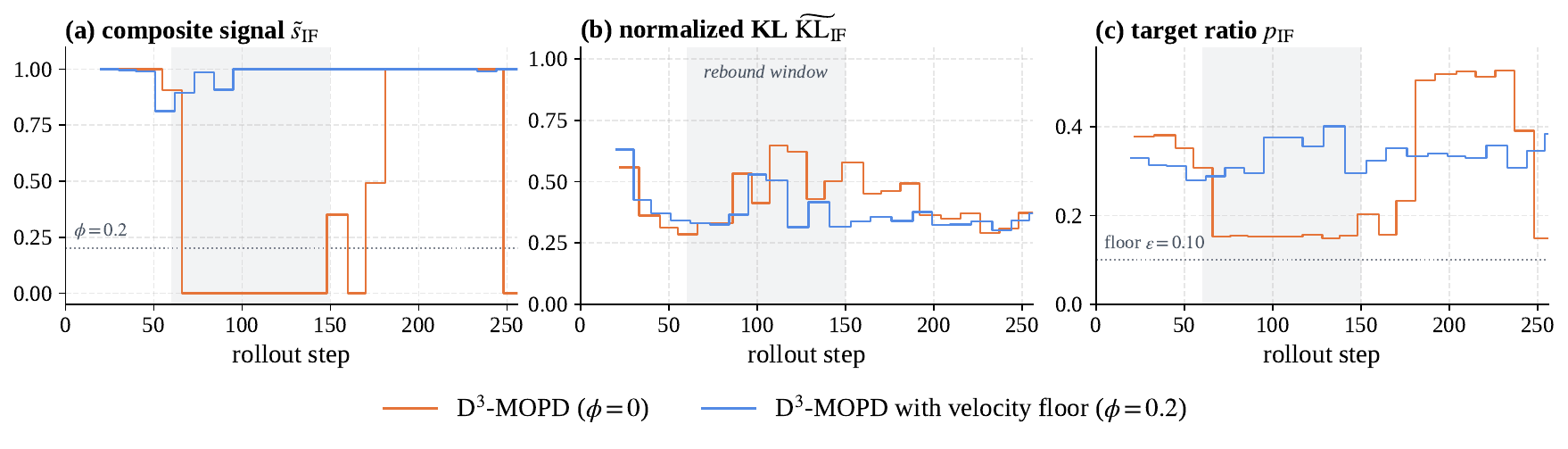}
    \vspace{-20pt}
    \caption{IF domain under $\phi=0.2$ (blue) and $\phi=0$ (orange); shaded band marks the rebound window. $\tilde s_\text{IF}=0$ means IF's velocity is clipped; $\tilde s_\text{IF}=1$ means IF has the strongest signal.}
    \label{fig:vfloor_if}
\end{figure}

\begin{figure}[h]
    \centering
    \includegraphics[width=0.55\linewidth]{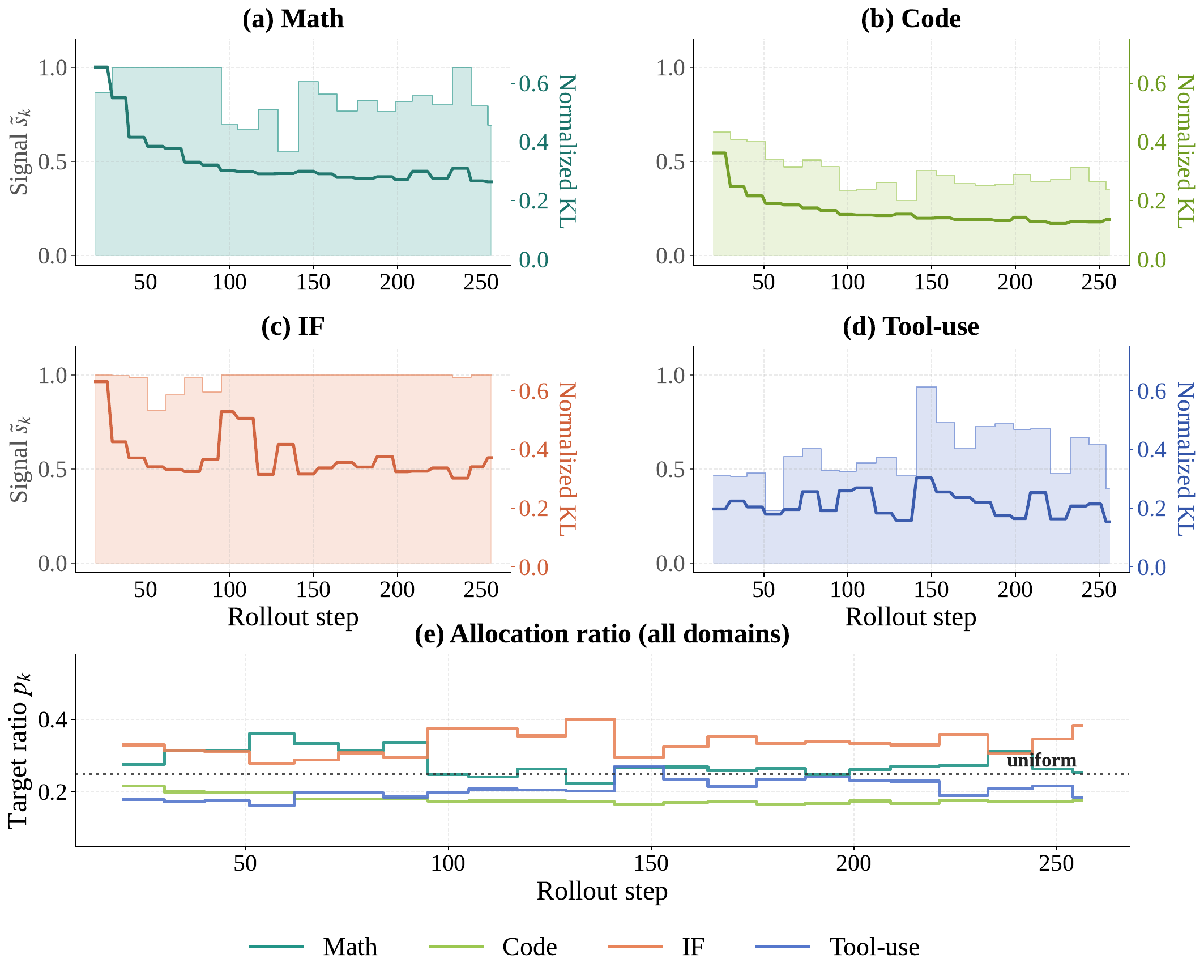}
    \vspace{-10pt}
    \caption{Per-domain composite signal and normalized KL across 256 rollout steps for the velocity floor trial with $\phi = 0.2$, and allocation ratio in panel~(e). Each allocation reflects both the remaining gap and the descent velocity: Math and IF descend from large gaps and receive the most budget, Code releases share as its descent slows, and Tool-use gains budget as its velocity rises later.}
    \label{fig:vfloor_signal_dynamics}
\end{figure}

Table~\ref{tab:vfloor} reports per-benchmark accuracy for vanilla MOPD, D$^3$-MOPD, and the velocity-floor variant. Adding $\phi = 0.2$ raises both Instruction-Following scores over D$^3$-MOPD (IFBench $+0.8$, IFEval $+0.3$), as the extra budget during IF's rebound window trains IF further. Elsewhere the variant stays close to D$^3$-MOPD (average $62.2$ vs.\ $62.3$). For a fair comparison the velocity-floor variant is reported at step $95$ to match D$^3$-MOPD, whereas vanilla MOPD and D$^3$-MOPD are shown at their best-average checkpoints (step $143$ and step $95$) selected from all $16$ evaluated checkpoints. Other checkpoints of the velocity-floor variant may exceed D$^3$-MOPD's average.
\begin{table}[h]
\centering
\small
\setlength{\tabcolsep}{1pt}
\caption{Per-benchmark accuracy for the velocity floor ablation. The variant with $\phi = 0.2$ scores highest on the shaded IF columns and LCB-v6, and stays close to D$^3$-MOPD elsewhere.}
\vspace{4pt}
\begin{tabular}{l cc >{\columncolor{gray!15}}c >{\columncolor{gray!15}}c cc c c}
\toprule
& \multicolumn{2}{c}{Math} & \multicolumn{2}{>{\columncolor{gray!15}}c}{Instruction Following} & \multicolumn{2}{c}{Code} & Tool-use & \\
\cmidrule(lr){2-3} \cmidrule(lr){4-5} \cmidrule(lr){6-7} \cmidrule(lr){8-8}
Method & AIME25 & HMMT(N) & IFBench & IFEval & LCB-v6 & OJB-C++ & BFCL(B) & Avg. \\
\midrule
Vanilla MOPD {\scriptsize (step 143)}                                        & 73.2          & 71.0          & 45.7          & 91.0          & 61.0          & 27.6          & 60.0          & 61.4 \\
\midrule
D$^3$-MOPD {\scriptsize (step 95)}                                           & \textbf{74.0} & \textbf{71.4} & 45.1          & 91.2          & 62.5          & \textbf{29.7} & \textbf{62.5} & \textbf{62.3} \\
\makecell[l]{D$^3$-MOPD w/ velocity floor \\ {\scriptsize ($\phi{=}0.2$, step 95)}}           & \textbf{74.0} & 71.3          & \textbf{45.9} & \textbf{91.5} & \textbf{62.7} & 29.3          & 61.0          & 62.2 \\
\bottomrule
\end{tabular}
\label{tab:vfloor}
\end{table}

\vspace{-10pt}
\paragraph{Discussion.} The velocity floor $\phi$ interpolates between two limit cases. At $\phi = 0$ the scheduler equals the paper formulation (Eq.~\ref{eq:descent_velocity}), which needs both a remaining gap and a positive descent velocity to invest in a domain. At $\phi = 1$ the composite reduces to the gap-only signal (the ``w/o velocity'' ablation in Table~\ref{tab:ablation}). Intermediate values such as $[0.1, 0.3]$ keep the product form and add a gap-proportional fallback whenever the velocity is clipped. This complements the rolling window count $R$: $R$ smooths noise within a single window, while $\phi$ handles cases where the velocity stays negative for longer than $RW$ consecutive steps.

\subsection{Skipping Rehearsal Domains}
\label{app:robustness_rehearsal}

\begin{figure}[t]
    \centering
    \includegraphics[width=\linewidth]{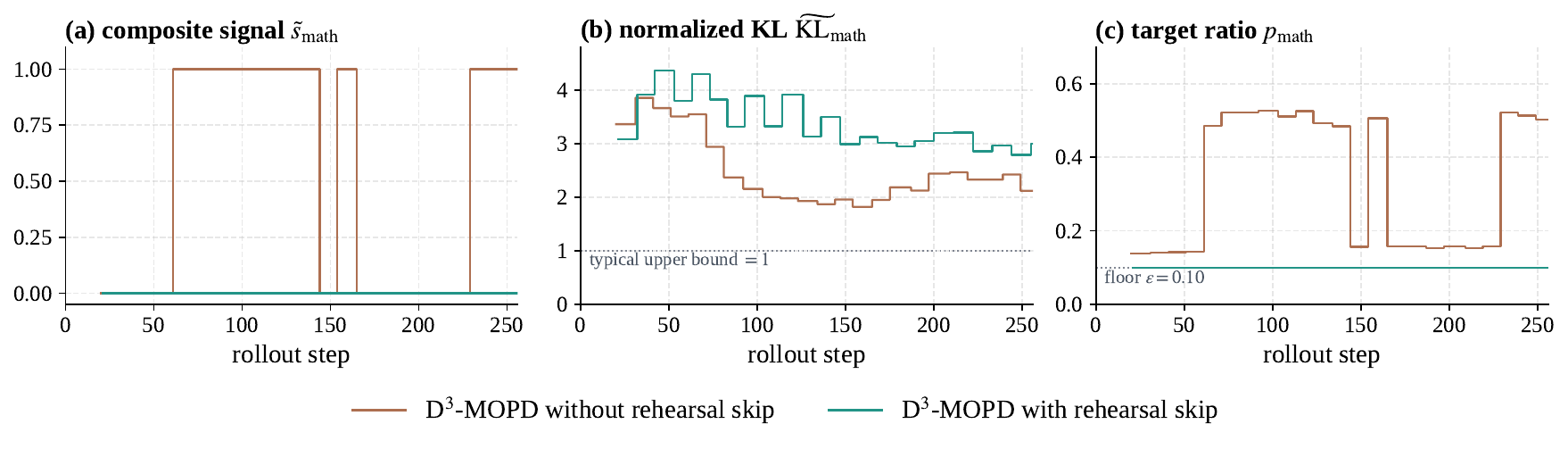}
    \vspace{-20pt}
    \caption{Math domain when Math uses the frozen Qwen3.6-35B-A3B student as its teacher, with rehearsal skip enabled (teal) and disabled (brown). $\tilde s_\text{math}=0$ means Math is pinned to the floor by rehearsal skip; $\tilde s_\text{math}=1$ means Math has the strongest signal.}
    \label{fig:rehmath_math}
\end{figure}

\paragraph{Failure mode.} In a model-merging scenario a teacher may be the frozen student itself, used to preserve the student's existing capability on that domain while the other domains are trained. Under this self-teacher setup $\mathrm{KL}_k^{(0)}$ is close to zero, and $\widetilde{\mathrm{KL}}_k = \overline{\mathrm{KL}}_k / \mathrm{KL}_k^{(0)}$ becomes very large for any small fluctuation in $\overline{\mathrm{KL}}_k$. The composite $s_k = \widetilde{\mathrm{KL}}_k \cdot v_k$ therefore inflates, and the watcher allocates a large budget to a domain that does not need learning, at the expense of the other domains.

\paragraph{Extension.} We add a rehearsal indicator $r_k \in \{0, 1\}$ that marks such a self-teacher domain. A domain is set to $r_k = 1$ either by an explicit configuration list or automatically when $\mathrm{KL}_k^{(0)} < \tau_{\text{abs}}$. The threshold $\tau_{\text{abs}}$ is set one or two orders of magnitude below the smallest expected non-rehearsal initial KL. Rehearsal domains are pinned to the floor and the remaining budget goes to non-rehearsal domains through the same softmax as before:
\begin{small}
\begin{equation}
    p_k(t) \;=\;
    \begin{cases}
        \epsilon, & r_k = 1, \\[2pt]
        \epsilon + \bigl(1 - K\epsilon\bigr)\dfrac{\exp(\tilde{s}_k(t)/T)}{\sum_{j:\,r_j=0}\exp(\tilde{s}_j(t)/T)}, & r_k = 0.
    \end{cases}
    \label{eq:rehearsal}
\end{equation}
\end{small}
Leaving the configuration list empty and $\tau_{\text{abs}} = 0$ disables both paths and reduces Eq.~\ref{eq:rehearsal} to Eq.~\ref{eq:ratio}.

\paragraph{Empirical validation.} We replace the Math teacher with the frozen Qwen3.6-35B-A3B student while keeping the Code, IF, and Tool teachers unchanged; other settings match the main experiment. Figure~\ref{fig:rehmath_math} compares Math under two variants: the baseline without rehearsal skip and the run with the skip enabled. In the baseline (a) the composite signal $\tilde s_\text{math}$ jumps between $0$ and $1$, and (c) $p_\text{math}$ climbs to $0.40$--$0.52$, well above the $1/K = 0.25$ share Math would deserve. Panel (b) explains the mechanism: $\widetilde{\mathrm{KL}}_\text{math}$ stays between $2$ and $4$, far above the typical upper bound of $1$, because Math's initial KL against the frozen-student teacher is near zero. With rehearsal skip enabled Math is pinned to $\epsilon = 0.10$ throughout, freeing budget for the other three domains.

\paragraph{Discussion.} The rehearsal indicator can be set in two ways. An explicit list of domain names covers the case where the user knows in advance which teachers are frozen copies of the student. An absolute-KL rule $\mathrm{KL}_k^{(0)} < \tau_{\text{abs}}$ covers the automatic case but requires $\tau_{\text{abs}}$ to be calibrated below the smallest expected non-rehearsal initial KL. Using both together catches mis-specifications from either side. A rehearsal domain's ratio is fixed at $\epsilon$ throughout training, which stops the initial-KL normalisation (Eq.~\ref{eq:normalized_kl}) from over-amplifying near-zero KL values.

\end{document}